\documentclass[11pt]{article}
\usepackage[margin=1in]{geometry}
\usepackage{amsmath,amssymb,amsthm}
\usepackage{algorithm}
\usepackage{algorithmic}
\usepackage{booktabs}
\usepackage[round]{natbib}
\usepackage{hyperref}
\usepackage{enumitem}
\usepackage{float}
\usepackage{multirow}
\usepackage{array}
\usepackage{graphicx}
\usepackage{tikz}
\usetikzlibrary{arrows.meta,positioning}
\usepackage[protrusion=true,expansion=false]{microtype}
\usepackage{subcaption}
\usepackage{setspace}
\usepackage[labelsep=period]{caption}
\usepackage{courier}
\usepackage{tabularx}
\usepackage{appendix}
\usepackage{makecell}
\usepackage{hyperref}

\hypersetup{colorlinks=true, linkcolor=blue, citecolor=blue, urlcolor=blue}

\newtheorem{theorem}{Theorem}
\newtheorem{proposition}{Proposition}

\newtheorem{definition}{Definition}
\newtheorem{remark}{Remark}

\DeclareMathOperator*{\argmin}{arg\,min}

\title{\large\bfseries Learning Auditable Classifier Models: Source-Disjoint Tree Ensembles}
\author{
	Srikumar Krishnamoorthy\\
	Information Systems Area, Indian Institute of Management Ahmedabad, India
}
\date{}

\begin{document}
	\maketitle
	
	\begin{abstract}
		Predictive models in clinical and regulated settings must be accurate and fully auditable. Tree ensembles deliver strong accuracy on tabular data, but their sequential boosting couples structure discovery with coefficient estimation, making compact per-prediction auditing difficult. Interpretable alternatives impose structural constraints that limit expressiveness: generalized additive models typically restrict interactions to pairwise terms and post-hoc rule extractors produce overlapping rules that hinder compact interpretation. We introduce Residual Pattern Tree Ensemble (RPTE), a three-stage learning approach, that is built on three key principles: bounded feature budget,  source disjointness, and separate coefficient estimation. Stage~1 builds a supervised symbolic feature vocabulary. Stage~2 grows shallow trees under a source-disjointness constraint, where each raw variable is allocated to at most one tree, and retains only the discovered tree structures. Stage~3 solves a single $\ell_1$-regularized logistic regression over leaf-region indicators, yielding jointly optimal sparse coefficients. This learning approach ensures that every prediction decomposes into an algebraic sum of named, non-overlapping rule contributions, enabling full auditability by design. Empirical evaluation on twelve clinical-domain binary classification benchmarks using repeated stratified 5-fold cross-validation shows that RPTE performs competitively against tuned opaque ensembles and interpretable baselines. RPTE reduces model inspection units by 9$\times$ to 87$\times$ relative to XGBoost and maintains lower audit complexity than EBM on all 12 datasets. RuleFit requires comparable or fewer inspection units on three datasets where its rule count is small, but without source-disjointness guarantees. The source code is available at \href{https://github.com/srikumar2050/hugiml-core}{this https URL}.
		
	\end{abstract}
	
	\section{Introduction}
	\label{sec:intro}
	
	In regulated domains such as healthcare, finance, and insurance, model outputs inform decisions that directly affect individual outcomes. Predictive models in these settings must satisfy two competing requirements: strong classification accuracy and sufficient transparency for regulatory review. Regulatory frameworks increasingly mandate that these outputs be traceable to verifiable decision logic \citep{rudin2019stop, rudin2022interpretable}. A growing body of work formalizes what interpretability means in practice \citep{murdoch2019definitions, arrieta2020xai}: a model is interpretable when its predictions can be verified, audited, and explained without relying on secondary approximation tools. This creates a practical challenge: the most accurate tabular classifiers are often the least auditable.
	
	Gradient-boosted tree ensembles \citep{friedman2001greedy} are the dominant paradigm for tabular classification. Implementations such as XGBoost \citep{chen2016xgboost} and LightGBM \citep{ke2017lightgbm} achieve strong accuracy on clinical and scientific benchmarks. Recent large-scale comparisons confirm that tree-based methods remain competitive with or superior to deep learning on tabular tasks \citep{grinsztajn2022tree, shwartzziv2022tabular, fernandezdelgado2014classifiers}. These models construct hundreds of shallow decision trees in sequence, each fitted to the residual error of its predecessors. The sequential construction that produces strong accuracy also creates an obstacle for deployment in regulated settings. No individual prediction can be traced to a single, readable set of decision rules.
	
	Consider a gradient-boosted model with 200 trees of depth~4. A single prediction aggregates the outputs of all 200 trees, with the same variable potentially appearing repeatedly under different split thresholds across trees. A clinical reviewer verifying why a patient was flagged as high-risk must reconstruct the relevant decision paths across many trees. In a cardiac risk model, for example, the variable ``oldpeak'' (exercise-induced ST depression) might appear in 30 separate trees at 15 different thresholds, each interacting with different co-variables. The model's overall decision logic may span thousands of split conditions and does not yield a self-contained rule for that prediction. Post-hoc explanation methods can summarize this behavior but do not expose the model's decision logic in rule form. SHAP \citep{lundberg2017unified} assigns feature contributions using Shapley values, defined through changes in expected model output relative to a chosen background distribution. LIME \citep{ribeiro2016why} fits a simple surrogate model to approximate the predictor's behavior in a neighborhood of the instance being explained. Even when computed faithfully, these methods provide feature attribution or local approximation rather than a direct representation of the ensemble's underlying decision structure. In high-stakes settings, this distinction motivates the use of models whose decision rules can themselves be inspected, verified, and held accountable \citep{rudin2019stop, freitas2014comprehensible}.
	
	We argue that one important source of opacity in boosted tree ensembles is their stagewise construction. At each boosting round, a new tree structure is constructed and numerical leaf values are estimated conditional on the ensemble fitted thus far. Once a tree is added, its structure and leaf values are fixed rather than jointly re-estimated. This leads to two key problems. First, the resulting leaf coefficients are not jointly optimal conditional on the collection of structures ultimately discovered. Second, decision logic becomes fragmented across the ensemble: the same variable may appear repeatedly, at different thresholds and in interaction with different variables, across many trees.
	
	Existing interpretable methods address parts of this problem but leave important gaps. Generalized additive models such as Explainable Boosting Machines (EBMs; \citealp{lou2013accurate}; InterpretML, \citealp{nori2019interpretml}) impose structural transparency by representing the prediction as a sum of univariate effects and selected pairwise interactions.\footnote{The InterpretML implementation supports higher-order interactions when the user explicitly specifies them as tuples of feature indices. Automatic detection is limited to pairs.} This architectural restriction means that interactions involving more than two variables cannot be discovered automatically. In a breast cancer diagnosis model, for example, the joint effect of nuclear size, boundary irregularity, and texture variation, a three-way interaction supported by clinical cytopathology, lies outside the model's representational capacity. Additionally, EBMs produce per-feature shape functions with many bins, leading to high audit complexity on feature-rich datasets. On the WDBC dataset (30 features), EBM generates over 12,000 model inspection units, in our experiments (Section~\ref{sec:auditLoad}).
	
	Rule-based classifiers take a different approach. RuleFit \citep{friedman2008predictive} extracts path-based rules from an ensemble and estimates a sparse linear combination using $\ell_1$-regularized regression \citep{tibshirani1996lasso}, separating rule generation from coefficient estimation. However, candidate rules are generated through greedy univariate tree splits, and interactions whose constituent variables provide little marginal gain at early splits may be missed. Sparsification operates over rules rather than their constituent variables, so the same variable can appear across multiple selected rules. When 10 to 20 rules fire simultaneously per prediction, with shared variables and potentially conflicting coefficient signs, the resulting model hinders compact interpretation. Ordered rule lists such as CORELS \citep{angelino2018corels} and Bayesian rule lists \citep{letham2015interpretable} avoid this simultaneous-activation problem, but their interpretability depends on list length, and the operational meaning of each rule depends on all preceding rules having failed. Sparse scoring systems \citep{ustun2016supersparse} achieve strong interpretability but are limited to linear decision boundaries.
	
	In this study, our primary motivation is to address the two identified problems through the Residual Pattern Tree Ensemble (RPTE), a \emph{three-stage learning} approach built around three key principles: (a) bounded vocabulary of symbolic features subject to a user-defined budget, (b) source-disjointness constraint on the discovered tree structures, and (c) separation of structure discovery from coefficient estimation. In the \emph{first stage}, a supervised pattern miner constructs an interpretable symbolic feature vocabulary. In the \emph{second stage}, shallow trees are grown sequentially against the current boosting residuals, subject to a constraint that reserves each raw variable for at most one accepted tree. Each tree is assigned temporary leaf values for residual updates, but these values are discarded once its structure is accepted. In the \emph{third stage}, terminal-leaf membership indicators from all accepted trees, together with eligible direct terms, form a common design matrix. A single $\ell_1$-regularized logistic regression estimates their coefficients jointly. This final refitting allows structures discovered early in the boosting process to receive small or zero coefficients once all discovered structures are considered jointly. The source-disjointness constraint further prevents distinct trees from encoding competing threshold structures over the same raw variables, thereby limiting the proliferation of overlapping rules that can arise in unconstrained ensembles. In order to assess the strengths of the proposed approach, we investigate the following key research questions:
	
	\begin{enumerate}[label=(RQ\arabic*),leftmargin=3.5em]
		\item Does separating structure discovery from coefficient estimation while enforcing source disjointness preserve competitive predictive accuracy relative to tuned opaque ensembles and interpretable baselines?
		\item What quantifiable reduction in audit complexity does the proposed approach achieve while maintaining this predictive accuracy, and does this reduction hold consistently across datasets with different characteristics?
	\end{enumerate}
	
	To address these research questions, our study makes the following contributions:
	\begin{enumerate}[label=(\arabic*),leftmargin=2.5em]
		\item We formalize RPTE as a three-stage learning approach and identify its three defining architectural principles: bounded feature budget, source disjointness, and separation of structure discovery from coefficient estimation. Together, these distinguish RPTE from standard boosting, interpretable baselines, and post-hoc rule extraction methods.
		\item We characterize five structural properties of RPTE and establish corresponding complexity bounds, providing a formal foundation for its auditability.
		\item We report an empirical evaluation on twelve clinical-domain binary classification datasets against three ensemble methods and two interpretable baselines. The evaluation includes pattern-level interpretability analysis, cross-fold and same-fold Rashomon analysis, eight ablation experiments, and a robustness study of strict source disjointness.
	\end{enumerate}
	
	The remainder of this paper is organized as follows. Section~\ref{sec:related} reviews related work. Section~\ref{sec:prelim} introduces notation and audit-load metrics. Section~\ref{sec:framework} presents the three stage learning approach. Section~\ref{sec:theory} establishes theoretical properties and complexity bounds. Section~\ref{sec:experiments} reports predictive performance, audit-load reduction, and ablation results. Section~\ref{sec:discussion} presents the interpretability analysis, Rashomon findings, practical implications, and limitations. Section~\ref{sec:conclusion} concludes with directions for future research.

	\section{Related Work}
	\label{sec:related}
	This section reviews four lines of research that bear on the problem of simultaneously achieving predictive accuracy and auditability in tabular classification. 
	
	\textit{Gradient-boosted ensembles.}
	Gradient boosting \citep{friedman2001greedy, friedman2000additive} constructs an additive model by sequentially fitting shallow trees to pseudo-residuals. $L_2$-boosting theory \citep{buhlmann2003boosting} establishes convergence for functional gradient descent. XGBoost \citep{chen2016xgboost} adds second-order gradient information and regularization. LightGBM \citep{ke2017lightgbm} introduces gradient-based one-side sampling. Random Forests \citep{breiman2001random} aggregate many independently randomized deep trees using bagging and random feature selection. These tree-ensemble methods share a structure-estimation coupling issue: each tree's leaf predictions are estimated locally and remain fixed thereafter. The proposed approach decouples these stages - discarding leaf predictions after structure discovery and re-estimating all leaf values jointly in a single convex program.
	
	\textit{Additive interpretable models.}
	Generalized additive models restrict each term to one variable. Explainable Boosting Machines (EBMs; \citealp{lou2013accurate}; InterpretML, \citealp{nori2019interpretml}) extend this to pairwise interactions via boosted CART. Automatic interaction detection is limited to pairs; higher-order terms require explicit user specification. Neural additive models \citep{agarwal2021nam} replace shape functions with small networks. These methods are transparent by construction but structurally limited. The proposed approach discovers interactions of arbitrary order through a pairwise-composition search (Section~\ref{sec:framework}), while the source-disjointness constraint ensures auditability.
	
	\textit{Rule-based interpretable classifiers.}
	RuleFit \citep{friedman2008predictive} extracts conjunctive rules from a Random Forest and applies lasso regression to leaf-path indicators. The differences from the proposed approach are threefold. First, RuleFit's rules inherit the marginal-screening limitation of CART splits, while the pairwise-composition search in Stage~2 discovers higher-order interactions at polynomial cost. Second, RuleFit imposes no source-disjointness constraint, so the same variable can appear in multiple rules. Third, RuleFit's lasso operates on indicators weighted by locally estimated leaf values, while the proposed approach operates on structural indicators with reassigned coefficients. FIGS \citep{tan2022figs} grows a small set of trees simultaneously but permits variable reuse across trees, sacrificing exact source-level decomposability. SIRUS \citep{benard2021sirus} stabilizes rule extraction by frequency thresholding but inherits the variable overlap of the underlying forest. Although each individual rule in these methods is readable, the fitted collections often hinders compact interpretation. Multiple rules can fire on the same observation with conflicting coefficient signs. The number of simultaneously active rules per prediction can reach 20 or more. Appendix~\ref{app:critique} presents an empirical analysis of these patterns.
	
	\textit{Ordered rule lists and sparse scoring systems.}
	A separate family represents predictions as ordered rule lists. CORELS \citep{angelino2018corels} uses branch-and-bound search to find a globally optimal rule list. SamRuLe \citep{pellegrina2024samrule} addresses computational cost by learning from a statistically selected sample. QCBA \citep{kliegr2023qcba} post-processes rule classifiers to recover information lost during discretization. Bayesian rule lists \citep{letham2015interpretable} generate decision lists via a pre-mined set of rules. Sparse scoring systems \citep{ustun2016supersparse} restrict coefficients to small integers for point-of-care use. While ordered rule lists avoid the simultaneous-activation problem, their interpretability depends on list length, antecedent complexity, and ordering dependence. The proposed approach takes a different path: rather than imposing rule ordering, it eliminates variable overlap through source disjointness.
	
	The idea of constructing a structural basis and then fitting a sparse model over it has precedents beyond RuleFit. Sparse additive models \citep{hastie2009elements} estimate components jointly under an $\ell_1$ penalty. \citet{dembczynski2008maximum} propose maximum likelihood rule ensembles. \citet{cohen1999simple} extract rules from decision lists via sequential covering. Optimal classification trees \citep{bertsimas2017optimal} and fast Rashomon sets \citep{molero2026rashomon} enumerate good models from the Rashomon set of near-optimal classifiers \citep{semenova2022existence, xin2022rashomon}. The proposed approach differs in two respects: the structural basis is built by residual-guided sequential search with a source-disjointness constraint, and leaf values are discarded entirely so that Stage~3 operates on structural indicators. The supervised pattern-mining component (Stage~1) relates to itemset mining \citep{agrawal1994fast} and its supervised extensions \citep{krishnamoorthy2024hugiml}, but the three-stage approach is not tied to any specific Stage~1 implementation.

	\begin{table}[h]
		\centering\small
		\caption{Comparison of structural properties across tabular classification methods.}
		\label{tab:comparison}
		\renewcommand{\arraystretch}{1.15}
		\setlength{\tabcolsep}{1.25pt}
		\begin{tabularx}{\textwidth}{
				@{}
				>{\raggedright\arraybackslash}p{3.5cm}
				*{7}{>{\centering\arraybackslash}X}
				@{}
			}
			
			\toprule
			Property
			& RPTE & RuleFit & EBM & FIGS & SIRUS & CORELS & XGBoost \\
			\midrule
			
			Source disjointness
			& Enforced
			& No
			& No
			& No
			& No
			& No
			& No \\
			
			Exact source decomposition
			& Yes
			& No
			& Term wise
			& No
			& No
			& N/A
			& No \\
			
			Separate structure/ weight fit
			& Yes
			& Yes
			& No
			& Optional refit
			& Fixed average
			& N/A
			& No \\
			
			Discovered interactions
			& \texttt{max depth}
			& Tree size
			& Pairwise$^*$
			& \texttt{max depth}
			& \texttt{max depth}
			& Pre-mined
			& \texttt{max depth} \\
			
			Bounded audit load
			& Hard
			& L1/soft
			& Interaction cap
			& Hard split cap
			& Hard rule cap
			& $\lambda$-implied
			& Indirect \\
			
			Native contribution unit
			& Leaf/ direct term
			& Rule/ linear term
			& Term score
			& Leaf
			& Rule output
			& First matching rule
			& Leaf \\
			
			\bottomrule
		\end{tabularx}
		\smallskip
		{\scriptsize
			\hspace*{-2em}$^*$EBM automatically selects pairwise interactions; higher-order
			interactions may be specified explicitly.
		}
	\end{table}
	
	Table~\ref{tab:comparison} distinguishes structural properties that affect how a fitted model can be inspected and audited. Source disjointness denotes an explicit construction constraint preventing an underlying source from being reused across independently fitted structural components. Exact source decomposition is stronger than algebraic additivity: contributions must correspond to non-redundant, source-owned components. Hence a model may expose exact native contributions without satisfying exact source decomposition; for example, overlapping RuleFit rules or FIGS trees may contribute in opposing directions to nearly the same observations (see Appendix~\ref{app:critique}). Separate structure/weight fit distinguishes methods that discover a representation before estimating its final coefficients from methods that learn structure and prediction values jointly; fixed aggregation and optional post-structure refitting are indicated separately. Bounded audit load characterizes whether and how the number of auditable structural units is controlled, distinguishing direct hard limits from sparsity penalties, regularization-implied bounds, or indirect complexity controls. Finally, native contribution unit reports the model component at which a prediction can be inspected directly; it does not imply unique raw-feature attribution.
	
	Viewed collectively, these structural differences expose three gaps in prior work. First, no existing method combines source disjointness with separated estimation to guarantee an exact, source-disjoint per-prediction decomposition. Second, existing interpretable model families do not jointly provide automatic discovery of higher-order interactions at polynomial cost while avoiding a fixed low-order interaction cap. Third, although individual methods can achieve low audit load or strong predictive performance, none provides formal bounds that tie audit load directly to user-specified structural budgets while retaining competitive predictive accuracy.
	
	\section{Preliminaries and Notations}
	\label{sec:prelim}
	
	Let $X \in \mathbb{R}^{n \times p_0}$ denote the raw input matrix with $n$ observations and $p_0$ variables, and let $y \in \{0,1\}^n$ denote the binary labels. Let $\mathcal{R}_0 = \{1, \ldots, p_0\}$ denote the set of raw variable indices. The logistic function is $\sigma(z) = 1/(1+e^{-z})$, and the binomial deviance loss used throughout is:
	\begin{equation}
		\mathcal{L}(F) = \frac{2}{n}\sum_{i=1}^{n}
		\bigl[-y_i \log \sigma(F_i) - (1-y_i)\log(1-\sigma(F_i))\bigr],
		\label{eq:deviance}
	\end{equation}
	where $F_i$ is the log-odds prediction for observation~$i$.
	
	\subsection{Feature vocabulary}
	\label{sec:vocab}
	
	A feature vocabulary $\mathcal{H}$ is constructed from the raw variables through up to three sources:
	\begin{enumerate}[label=(\roman*),topsep=2pt,itemsep=1pt]
		\item \textbf{Raw features:} the original $p_0$ columns of $X$.
		\item \textbf{Mined patterns:} up to $K$ binary indicator columns, each encoding a conjunctive condition over at most two raw variables. Any supervised pattern mining procedure (e.g., utility-based mining, frequent pattern mining) may serve as the constructor.
		\item \textbf{Augmented pairs} (optional): up to $K$ continuous columns computed as pairwise arithmetic transforms of raw variables (e.g., $|x_a - x_b|$, $x_a \cdot x_b$).
	\end{enumerate}
	The user-specified feature budget $K$ bounds the mined pattern and augmented pair counts independently. The total vocabulary size is
	\begin{equation}
		V \;=\; |\mathcal{H}| \;\leq\; p_0 + 2K.
		\label{eq:vocabsize}
	\end{equation}
	
	\begin{definition}[Source ownership]
		\label{def:source}
		Each column $j \in \mathcal{H}$ carries a source set $\omega(j) \subseteq \mathcal{R}_0$ with $|\omega(j)| \leq 2$, identifying the raw variables from which it is derived. A raw feature $k$ has $\omega(k)=\{k\}$; a mined pattern or augmented pair derived from variables $a$ and $b$ has $\omega(j) = \{a,b\}$. Two columns are \emph{source-disjoint} when $\omega(j) \cap \omega(j') = \emptyset$.
	\end{definition}
	
	When cross-validation is used, vocabulary construction is performed inside each training fold to prevent label leakage.
	
	\subsection{Structural parameters}
	\label{sec:params}
	
	Three parameters govern model construction: the maximum number of boosting rounds $T$, the maximum number of leaves per tree $L_{\max}$, and the maximum tree depth $D_{\max}$. After fitting, $M \leq T$ trees are accepted; source disjointness (Section~\ref{sec:framework}) further tightens this to $M \leq \min(T, p_0)$, since each accepted tree claims at least one raw variable. The set $D \subseteq \mathcal{H}$ consists of vocabulary columns not used by any accepted tree.
	
	\subsection{Audit-load metrics}
	\label{sec:audit-metrics}
	
	To quantify the audit complexity that any classification model produces, we define three levels of measurement that apply uniformly across all model architectures.
	
	\begin{definition}[Audit load]\label{def:audit}~
		\begin{enumerate}[label=(\roman*),topsep=2pt,itemsep=1pt]
			\item \textbf{Model units (MU):} non-zero coefficients in the final model.
			\item \textbf{Model inspection units (MIU):} total split conditions across all active leaf paths plus active direct terms, i.e., the number of conditions a reviewer must read to audit the complete model.
			\item \textbf{Instance inspection units (IIU):} mean conditions consulted per prediction over test instances, i.e., the per-prediction audit cost.
		\end{enumerate}
	\end{definition}
	
	MIU is the primary model-level metric. These metrics are computed for all methods using the same counting procedure. We note that MIU and IIU measure auditor workload (conditions to verify), but they do not capture cognitive difficulty of individual conditions or domain familiarity.
	
	\begin{table}[!t]
		\caption{Summary of notation.}
		\label{tab:notation}
		\begin{tabular}{@{}cl@{}}
			\toprule
			Symbol & Description \\
			\midrule
			$n, p_0$ & Number of observations and raw variables \\
			$K$ & User-specified feature budget \\
			$\mathcal{H}$ & Feature vocabulary, $V = |\mathcal{H}| \leq p_0 + 2K$ \\
			$\omega(j)$ & Source set of column $j$, $\omega(j) \subseteq \mathcal{R}_0$, $|\omega(j)| \leq 2$ \\
			$T$ & Maximum boosting rounds \\
			$L_{\max}$ & Maximum leaves of tree \\
			$D_{\max}$ & Maximum depth of tree \\
			$M$ & Number of accepted trees, $M \leq \min(T, p_0)$ \\
			$D$ & Set of unused vocabulary columns (direct terms) \\
			$q$ & Design matrix width, $q \leq M \cdot L_{\max} + |D|$ \\
			$\rho$ & Ridge parameter for Newton leaf values \\
			$\eta_0$ & Initial learning rate for backtracking \\
			$\lambda_1$ & $\ell_1$ regularization parameter \\
			\bottomrule
		\end{tabular}
	\end{table}
	
	\smallskip
	Table~\ref{tab:notation} summarizes the notation used throughout.
	
	\section{Methodology}
	\label{sec:framework}
	This section presents the three-stage learning approach and its concrete realization, RPTE (Residual Pattern Tree Ensemble).
	
	\subsection{Architecture overview}
	
	Standard gradient boosting merges three distinct questions: what primitives should the learner use, how should those primitives be composed into tree structures, and how much predictive weight should each structure receive. The proposed three-stage approach separates these questions explicitly.
	
	\begin{center}
		\begin{tikzpicture}[
			node distance=0.9cm and 1.6cm,
			box/.style={rectangle, draw=black, rounded corners=4pt,
				minimum width=3.4cm, minimum height=1.0cm,
				font=\small, align=center, thick},
			arr/.style={-{Latex[length=3.5mm]}, thick, line width=1.2pt}
			]
			\node[box]  (S1) {Stage 1\\Vocabulary};
			\node[box, right=of S1] (S2) {Stage 2\\Structure Discovery};
			\node[box, right=of S2]  (S3) {Stage 3\\Global Estimation};
			\draw[arr]   (S1) -- (S2);
			\draw[arr] (S2) -- (S3);
			\node[font=\scriptsize, below=0.5cm of S1, align=center]
			{Supervised mining\\provenance-annotated columns};
			\node[font=\scriptsize, below=0.5cm of S2, align=center]
			{Source-disjoint trees\\leaf values discarded};
			\node[font=\scriptsize, below=0.5cm of S3, align=center]
			{$\ell_1$ logistic regression\\jointly optimal coefficients};
		\end{tikzpicture}
	\end{center}
	
	\textbf{Stage~1} answers: \emph{what primitives may the learner use?}
	A supervised procedure constructs the feature vocabulary $\mathcal{H}$ with $V \leq p_0 + 2K$ columns, as described in Section~\ref{sec:vocab}. Each column carries a provenance record $\omega(j) \subseteq \mathcal{R}_0$. The subsequent stages are agnostic to the choice of mining algorithm.
	
	\textbf{Stage~2} answers: \emph{how should the primitives be composed?}
	Shallow trees grow sequentially against pseudo-residuals of the binomial deviance~\eqref{eq:deviance}. At each round~$m$, the pseudo-residuals and Newton weights are
	\begin{equation}
		r_i = y_i - \sigma(F_i), \qquad w_i = \sigma(F_i)\bigl(1 - \sigma(F_i)\bigr).
		\label{eq:residuals}
	\end{equation}
	A tree $h_m$ is grown using only columns whose raw sources remain unclaimed. Each leaf $\ell$ receives a Newton step
	\begin{equation}
		v_\ell = \frac{\sum_{i \in \ell} r_i}{\sum_{i \in \ell} w_i + \rho},
		\label{eq:newton}
	\end{equation}
	where $\rho > 0$ is a ridge parameter. A backtracking line search finds $\eta_m \leq \eta_0$ satisfying $\mathcal{L}(F_{m-1} + \eta_m h_m) < \mathcal{L}(F_{m-1})$, starting from an initial rate $\eta_0$ and halving up to $B_{\mathrm{ls}}$ times. On acceptance, the raw sources $\bigcup_{j \in \mathrm{used}(h_m)} \omega(j)$ are removed from the eligible pool. The temporary leaf values are then discarded; only the tree structure (the partition into leaf regions) is retained.
	
	\textbf{Stage~3} answers: \emph{how much predictive weight should each structure receive?}
	Leaf-membership indicators from all $M$ accepted trees, together with any unused vocabulary columns $D$, form a design matrix $Z \in \{0,1,\mathbb{R}\}^{n \times q}$, where $q \leq M \cdot L_{\max} + |D|$. A single $\ell_1$-penalized logistic regression assigns jointly optimal coefficients:
	\begin{equation}
		(\beta^*, \gamma^*) = \argmin_{\beta,\gamma}
		\bigl[\mathcal{L}(\beta,\gamma;Z,y) + \lambda_1(\|\beta\|_1 +
		\|\gamma\|_1)\bigr],
		\label{eq:stage3}
	\end{equation}
	where $\beta$ collects the leaf-region coefficients and $\gamma$ collects the direct-term coefficients. The $\ell_1$ penalty \citep{tibshirani1996lasso} drives a fraction of coefficients to zero, further sparsifying the model. The resulting prediction for any input $x$ decomposes as
	\begin{equation}
		F(x) = \beta_0
		+ \underbrace{\sum_{t=1}^{M} \beta_{t,\ell_t(x)}}_{\text{tree contributions}}
		+ \underbrace{\sum_{j \in D} \gamma_j \phi_j(x)}_{\text{direct terms}},
		\label{eq:decomp}
	\end{equation}
	where $\ell_t(x)$ is the unique leaf of tree~$t$ reached by~$x$ and $\phi_j(x)$ is the value of vocabulary column $j$ for input $x$.
	
	\subsection{Design principles}
	
	The approach is grounded in three design principles that together constrain model construction and ensure that the resulting model is auditable without post-hoc attribution tools.
	
	\textbf{Bounded feature budget (P1).}
	The vocabulary size $V \leq p_0 + 2K$ is bounded by the user-specified budget $K$ (Equation~\ref{eq:vocabsize}). Combined with the tree parameter bounds ($T$ rounds, $L_{\max}$ leaves, $D_{\max}$ depth), the maximum audit load can be computed before training begins (Theorem~\ref{thm:complexity}). Unlike opaque ensembles, where model complexity is an emergent property of the training process, the audit cost here is deterministic.
	
	\textbf{Source disjointness (P2).}
	Each raw variable is owned by at most one tree. Once a tree uses a column whose provenance includes variable $k$, no subsequent tree may use any column derived from $k$. This forces the model to allocate each raw variable to one coherent structural role and limits the ensemble to at most $p_0$ trees. A direct consequence is exact decomposability: every prediction is an algebraic sum of terms with non-overlapping raw-variable sources, as expressed in Equation~\eqref{eq:decomp}. A reviewer can verify any prediction by reading each active leaf path and summing the corresponding coefficients.
	
	\textbf{Separated coefficient estimation (P3).}
	Structure discovery (Stage~2) and coefficient estimation (Stage~3) are fully separated. Stage~2 uses Newton boosting with temporary leaf values~\eqref{eq:newton} to discover tree structures. After each tree is accepted, the temporary leaf values are discarded. Stage~3 then assigns each leaf region a coefficient that is jointly optimal over all discovered structures and unused vocabulary columns. This ensures that locally estimated coefficients from early boosting rounds, when residuals differ substantially from the final residuals, do not persist into the final model.
	
	\begin{algorithm}[h]
		\caption{{\fontfamily{pcr}\selectfont Residual Pattern Tree Ensemble (RPTE)}}
		\label{alg:fit}
		\begin{algorithmic}[1]
			\REQUIRE Raw data $X\in\mathbb{R}^{n\times p_0}$, labels $y$,
			feature budget $K$, structural bounds
			$(T{=}10,L_{\max}{=}12,D_{\max}{=}8)$,
			hyperparameters $(\eta_0{=}0.3,\lambda_1)$
			\ENSURE Tree structures and final coefficients $(\beta^*,\gamma^*)$
			
			\STATE \textbf{Stage 1:}
			$(\mathcal{H},\omega)\gets$
			{\fontfamily{pcr}\selectfont ConstructVocabulary}$(X,y,K)$
			\COMMENT{Supervised pattern mining}
			
			\STATE $F\gets\log(\bar y/(1-\bar y))\cdot\mathbf{1}$;\;
			$\mathrm{avail}\gets\mathcal{R}_0$;\;
			$\mathrm{trees}\gets\emptyset$;\;
			$m\gets1$;\;
			$\mathrm{success}\gets\mathrm{true}$
			
			\STATE \textbf{Stage 2:} Constrained sequential structure discovery
			\WHILE{$m\leq T$ \AND $\mathrm{success}$}
			\STATE $\mathrm{success}\gets\mathrm{false}$
			\STATE $r_i\gets y_i-\sigma(F_i)$;\;
			$w_i\gets\sigma(F_i)(1-\sigma(F_i))$;\;
			$z_i\gets\mathbb{I}(r_i>0)$;\;
			$\eta_m\gets\eta_0$; $\mathcal{E}\gets
			\{j\in\mathcal{H}:\omega(j)\subseteq\mathrm{avail}\}$
			
			\IF{$\mathcal{E}\neq\emptyset$ \AND $z$ is nonconstant}
			\STATE Fit a tree $h_m$ using the columns in $\mathcal{E}$,
			with targets $z$ and subject to $L_{\max}$ and $D_{\max}$
			
			\IF{$h_m$ has a valid split}
			\STATE Compute Newton leaf values via~\eqref{eq:newton}
			\STATE Repeatedly halve $\eta_m$ until
			$F+\eta_mh_m$ decreases binomial deviance
			(at most $B_{\mathrm{ls}}{=}6$ trials)
			
			\IF{a decreasing step was found}
			\STATE $F\gets F+\eta_mh_m$;\;
			$\mathrm{avail}\gets
			\mathrm{avail}\setminus
			\displaystyle\bigcup_{j\in\mathrm{used}(h_m)}\omega(j)$
			\STATE Retain the structure of $h_m$ for the final representation
			\STATE $\mathrm{trees}\gets\mathrm{trees}\cup\{h_m\}$;\;
			$m\gets m+1$;\;
			$\mathrm{success}\gets\mathrm{true}$
			\ENDIF
			\ENDIF
			\ENDIF
			\ENDWHILE
			
			\STATE \textbf{Stage 3:} Global coefficient estimation
			\STATE $D\gets
			\{j\in\mathcal{H}:
			j\notin\displaystyle\bigcup_{h\in\mathrm{trees}}\mathrm{used}(h)\}$
			
			\STATE Build $Z$ from leaf indicators of $\mathrm{trees}$ and columns in $D$
			
			\STATE Solve~\eqref{eq:stage3} via $\ell_1$-regularized logistic
			regression for $(\beta^*,\gamma^*)$
			
			\RETURN tree structures, $\beta^*$, $\gamma^*$
		\end{algorithmic}
	\end{algorithm}
	
	\subsection{Algorithm pseudo-code}
	
	Algorithm~\ref{alg:fit} summarizes the complete RPTE learning process in three stages. Lines~1--2 construct the supervised vocabulary $\mathcal{H}$, retain its raw-source mapping $\omega$, and initialize the additive score and available-source set. Lines~3--19 perform sequential structure discovery. At each round, RPTE computes logistic residuals and Newton weights, forms residual-sign targets, and fits a depth- and leaf-constrained tree using only columns whose raw sources remain available. Newton leaf values are used provisionally to determine, through backtracking, whether the candidate tree decreases binomial deviance. An accepted tree updates the additive score, and every raw source used by that tree is removed from the pool available to subsequent trees. If no eligible split or deviance-reducing step is found, the success condition remains false and tree construction terminates. Finally, Lines~20--24 form a representation from the retained tree-leaf indicators and unused direct features, and jointly estimate their final coefficients using $\ell_1$-regularized logistic regression. Thus, the stage-wise Newton values guide structure discovery but are not carried into the final coefficient model.
	
	\section{Theoretical Analysis}
	\label{sec:theory}
	This section establishes structural properties of RPTE and derives complexity bounds.

	\begin{proposition}[Monotone descent]
		\label{prop:mono}
		The training deviance is strictly decreasing across Stage~2 rounds: $\mathcal{L}(F_m) < \mathcal{L}(F_{m-1})$ for every accepted round $m$.
	\end{proposition}
	
	\begin{proof}
		Each Stage~2 update takes the form $F_m = F_{m-1} + \eta_m h_m$. The tree $h_m$ is accepted only when the backtracking line search identifies $\eta_m > 0$ satisfying $\mathcal{L}(F_{m-1} + \eta_m h_m) < \mathcal{L}(F_{m-1})$. If no such $\eta_m$ exists, the tree is rejected and Stage~2 terminates.
	\end{proof}
	
	\begin{proposition}[Finite termination]
		\label{prop:term}
		Algorithm~\ref{alg:fit} accepts at most $\min(T, p_0)$ rounds.
	\end{proposition}
	
	\begin{proof}
		Every accepted round uses at least one column $j$ with $\omega(j) \neq \emptyset$. The raw sources of $j$ are removed from $\mathrm{avail}$. Since $|\mathrm{avail}|$ starts at $p_0$ and decreases by at least one per accepted round, at most $p_0$ rounds can be accepted. The budget $T$ provides a second bound.
	\end{proof}

	\begin{proposition}[Source disjointness]
		\label{prop:disj}
		Let $\Omega(h_t) := \bigcup_{j \in \mathrm{used}(h_t)} \omega(j)$
		denote the raw-variable sources used by an accepted tree $h_t$.
		Then any two distinct accepted trees are source-disjoint:
		\[
		\Omega(h_t)\cap\Omega(h_{t'})=\emptyset,
		\qquad t\neq t'.
		\]
	\end{proposition}
	
	\begin{proof}
		We proceed by induction over construction rounds. At the start of round
		$m$, maintain the invariant
		\[
		\mathrm{avail}
		=
		\mathcal{R}_0
		\setminus
		\bigcup_{\substack{s<m\\ h_s\ \mathrm{accepted}}}
		\Omega(h_s),
		\]
		where $\mathcal{R}_0$ is the initial set of available raw-variable
		sources. The invariant holds trivially at $m=1$.
		
		During round $m$, every column $j$ eligible for use by the candidate tree
		satisfies $\omega(j)\subseteq\mathrm{avail}$. Hence, if $h_m$ is accepted,
		\[
		\Omega(h_m)\subseteq\mathrm{avail},
		\]
		which is disjoint from the sources claimed by all previously accepted
		trees. The sources in $\Omega(h_m)$ are then removed from
		$\mathrm{avail}$ before the next round. Rejected trees do not modify
		$\mathrm{avail}$. The invariant is therefore preserved, proving pairwise
		source disjointness of all accepted trees.
	\end{proof}

	\begin{proposition}[Exact source-separated prediction decomposition]
		\label{prop:decomp}
		For every input $x$, Equation~\eqref{eq:decomp} is an algebraic identity.
		If $M$ trees are accepted and $D$ denotes the retained direct terms, each
		prediction contains at most $M+|D|$ active contributions: exactly one leaf
		contribution from each accepted tree and at most one contribution from
		each direct term.
		
		The tree contributions are mutually source-disjoint, and no direct term
		uses a raw-variable source claimed by an accepted tree. Consequently, the
		default model gives an exact source-separated decomposition between
		accepted trees and the direct-term channel.
		
		Under the strict direct-term ownership variant, distinct retained direct
		terms are additionally required to have disjoint raw-variable ownership.
		In that case, the decomposition is fully source-disjoint: no raw-variable
		source appears in more than one active contribution.
	\end{proposition}
	
	\begin{proof}
		The leaves of each accepted tree partition its input space, so for every
		input $x$ exactly one leaf indicator is active in each tree. Summing these
		$M$ leaf contributions together with the retained direct terms yields
		Equation~\eqref{eq:decomp}; hence at most $M+|D|$ terms contribute to any
		prediction.
		
		By Proposition~\ref{prop:disj}, the raw-variable sources associated with
		distinct accepted trees are disjoint. By construction, direct terms are
		restricted to sources not claimed by any accepted tree. Thus the tree
		contributions are mutually source-disjoint and are source-separated from
		the direct-term channel.
		
		In the default construction, distinct direct terms may still share a raw
		source. The strict direct-term ownership variant forbids such reuse. Under
		this additional restriction, the source sets of every pair of distinct
		contributing components are disjoint, establishing full term-level source
		disjointness.
	\end{proof}

	\begin{proposition}[Per-round search complexity]
		\label{prop:search}
		For fixed $L_{\max}$ and $D_{\max}$, the cost of growing one tree in Algorithm~\ref{alg:fit} is $O(V \cdot n)$, where $V = |\mathcal{H}|$ is the vocabulary size and $n$ is the number of observations.
	\end{proposition}
	
	\begin{proof}
		At each leaf expansion, the greedy search evaluates at most
		$|\mathcal{E}| \leq V$ candidate columns. Evaluating one column requires
		a scan of the observations reaching that leaf and computation of
		information gain, costing $O(n_\ell)$. Hence an expansion of leaf
		$\ell$ costs $O(Vn_\ell)$.
		
		Across all expansions at any fixed depth, the corresponding leaves
		contain disjoint subsets of the observations, so their total size is
		at most $n$. Since the tree has at most $D_{\max}$ depth levels, the
		total search cost is therefore $O(D_{\max} V n)$. With $D_{\max}$
		fixed, this simplifies to $O(Vn)$.
	\end{proof}
	
	\begin{theorem}[Model complexity bound]
		\label{thm:complexity}
		Let $M \leq \min(T, p_0)$ be the number of accepted trees and $L_{\max}$ the maximum number of leaves per tree. Then:
		\begin{enumerate}[label=(\alph*),topsep=2pt,itemsep=1pt]
			\item The number of leaf-region indicators in the design matrix $Z$ is at most $M \cdot L_{\max}$.
			\item The total model inspection units satisfy $\mathrm{MIU} \leq M \cdot L_{\max} \cdot D_{\max} + |D|$.
			\item After $\ell_1$ sparsification, there exists an optimal solution with at most $\mathrm{rank}(Z) \leq \min(n,q)$ active terms, where $q = M \cdot L_{\max} + |D|$.
		\end{enumerate}
	\end{theorem}
	
	\begin{proof}
		(a) Each of $M$ trees has at most $L_{\max}$ leaves.
		(b) Each leaf path has at most $D_{\max}$ conditions. Summing over all leaves of all trees gives $M \cdot L_{\max} \cdot D_{\max}$ conditions from trees, plus $|D|$ direct terms.
		(c) Among the $\ell_1$-regularized logistic regression optima, there
		exists a solution whose active columns are linearly independent.
		Consequently, its number of active terms is at most
		$\mathrm{rank}(Z) \leq \min(n,q)$.
	\end{proof}
	
	\begin{remark}[Audit-load comparison with opaque ensembles]
		\label{rem:ratio}
		The theoretical MIU bound of RPTE is $M \cdot L_{\max} \cdot D_{\max} + |D|$, which is polynomial in user-specified parameters. For a gradient-boosted ensemble with $T_x$ trees, each with at most $L_x$ leaves and depth $d_x$, the corresponding bound is $T_x \cdot L_x \cdot d_x$. Both expressions are structurally similar, but RPTE's source disjointness constrains $M$ to at most $\min(T, p_0)$, while opaque ensembles place no comparable bound on $T_x$. In practice, $T_x$ is often hundreds or thousands of trees, producing MIU values orders of magnitude larger. The empirical comparison in Section~\ref{sec:experiments} quantifies this gap on all twelve benchmarks.
	\end{remark}
	
	\begin{theorem}[Stage~3 optimality]
		\label{thm:stage3}
		The Stage~3 objective~\eqref{eq:stage3} is strictly convex in $(\beta, \gamma)$ when $Z$ has full column rank. The coordinate descent solver converges to the unique global minimum.
	\end{theorem}
	
	\begin{proof}
		Let $\theta=(\beta,\gamma)$ denote the full Stage~3 coefficient vector,
		so that the fitted logits can be written as $F=Z\theta$.
		The logistic loss is convex, with Hessian
		\[
		\nabla_{\theta}^{2}\mathcal{L}(\theta)
		=
		Z^\top W Z,
		\qquad
		W=\operatorname{diag}\!\left(
		\sigma(F_i)\bigl(1-\sigma(F_i)\bigr)
		\right).
		\]
		For finite $F_i$, each diagonal entry satisfies $W_{ii}>0$.
		Therefore, if $Z$ has full column rank, then for every nonzero vector
		$v$,
		\[
		v^\top Z^\top W Z v
		=
		(Zv)^\top W(Zv)
		>0,
		\]
		so $Z^\top W Z$ is positive definite and the logistic loss is strictly
		convex in $\theta$.
		
		Adding the convex $\ell_1$ penalty preserves strict convexity.
		Hence the Stage~3 objective has a unique global minimizer, and
		coordinate descent converges to this minimizer
		\citep{tseng2001convergence}. When $Z$ does not have full column rank, coefficient minimizers need not be unique, but all yield the same predicted probabilities.
	\end{proof}

	\section{Experimental Evaluation}
	\label{sec:experiments}
	\subsection{Experimental design}
	
	We evaluate RPTE on twelve binary classification datasets from PMLB \citep{romano2021pmlb} and OpenML, spanning sample sizes from $n = 87$ to $n = 11{,}500$ and feature dimensionalities from $p = 3$ to $p = 178$ (Table~\ref{tab:datasets}). Datasets were selected to cover three dimensions of difficulty: varying sample sizes (from $n{<}100$ to $n{>}10{,}000$), varying feature counts (3 to 178), and varying imbalance levels (balanced to 4:1 ratio).
	
	\begin{table}[!h]
		\centering\small
		\caption{Dataset summary with class balance. Balance shows the minority class proportion; ratio shows the majority-to-minority class ratio.}
		\label{tab:datasets}
		\renewcommand{\arraystretch}{1.06}
		\begin{tabularx}{\textwidth}{@{}l l r r r r X@{}}
			\toprule
			Dataset & Short & $n$ & $p$ & Balance & Ratio & Domain \\
			\midrule
			lupus           & lupus    &     87 &  3 & 0.40 & 1.5:1 & SLE diagnosis          \\
			postop.\ patient data & postop.  &     88 &  8 & 0.27 & 2.7:1 & Discharge planning     \\
			appendicitis    & append.  &    106 &  7 & 0.20 & 4.0:1 & Appendicitis diagnosis \\
			corral          & corral   &    160 &  6 & 0.44 & 1.3:1 & Structured interaction \\
			hepatitis       & hepatitis&    155 & 19 & 0.21 & 3.8:1 & Hepatitis survival     \\
			prnn\_crabs     & crabs    &    200 &  7 & 0.50 & 1.0:1 & Morphological class.   \\
			biomed          & biomed   &    209 &  8 & 0.36 & 1.8:1 & Biomedical screening   \\
			heart\_c        & heart-c  &    303 & 13 & 0.46 & 1.2:1 & Cardiac diagnosis      \\
			haberman        & haberman &    306 &  3 & 0.26 & 2.8:1 & Breast cancer survival \\
			saheart         & saheart  &    462 &  9 & 0.35 & 1.9:1 & Heart disease (SA)     \\
			WDBC            & WDBC     &    569 & 30 & 0.37 & 1.7:1 & Breast cancer diagnosis\\
			EEG-Seizure     & EEG seizure & 11{,}500 & 178 & 0.20 & 4.0:1 & Seizure detection  \\
			\bottomrule
		\end{tabularx}
	\end{table}
	
	All experiments use stratified 5-fold cross-validation with 3 repeats, yielding 15 splits per dataset. Inner hyperparameter selection uses stratified 3-fold CV on the training fold. We compare RPTE against five baselines: XGBoost, LightGBM, Random Forest (RF), EBM, and RuleFit. All six models use a uniform grid of 8 hyperparameter configurations each.%
	\footnote{Source code: \url{https://github.com/srikumar2050/hugiml-core}.} For specific configuration details, refer to Appendix~\ref{app:params}. The primary metrics are ROC~AUC, balanced accuracy, Brier score, and MIU/IIU, all computed as means over 15 splits.

	\subsection{Predictive performance}
	
	Table~\ref{tab:performance} reports mean ROC~AUC over 15 splits. Bold marks the highest AUC per row.
	
	Across all twelve datasets, RPTE achieves a mean AUC of 0.830. The opaque baselines achieve mean AUCs of 0.838 (XGBoost), 0.836 (LightGBM), and 0.837 (RF). EBM leads among interpretable methods at 0.853, followed by RuleFit at 0.831 and RPTE at 0.830. On corral, a dataset requiring interaction recovery, RPTE achieves perfect AUC. On EEG-Seizure (\(n=11{,}500\), \(p=178\)), RPTE achieves a ROC AUC of \(0.990\), only \(0.006\) below LightGBM.
	
	Panel-level Wilcoxon signed-rank tests (Table~\ref{tab:wilcoxon}) assess whether the AUC differences are statistically significant across the 12 dataset means. After Holm correction across the five prespecified RPTE-versus-baseline comparisons, RPTE performs significantly below XGBoost ($p_{\mathrm{Holm}}=0.005$) and EBM ($p_{\mathrm{Holm}}=0.039$). Differences between RPTE and LightGBM, RF, and RuleFit are not significant ($p_{\mathrm{Holm}}=0.442$, $0.467$, and $0.700$, respectively). These results position RPTE as competitive with several baselines: the median AUC gap to RuleFit is 0.001, and the gap to LightGBM is 0.004.
	
	\begin{table}[h]
		\centering\small
		\caption{Predictive performance: ROC~AUC, mean over 15 outer splits. The final row is the unweighted mean of the 12 dataset-level means. Bold marks the highest AUC within each row.}
		\label{tab:performance}
		\renewcommand{\arraystretch}{1.08}
		\begin{tabularx}{\textwidth}{@{}l *{6}{>{\centering\arraybackslash}X}@{}}
			\toprule
			Dataset & RPTE & XGBoost & LightGBM & RF & EBM & RuleFit \\
			\midrule
			lupus      & 0.768 & 0.772 & 0.775 & 0.754 & \textbf{0.806} & 0.730 \\
			postop.    & 0.432 & 0.440 & 0.426 & 0.333 & 0.407 & \textbf{0.443} \\
			append.    & 0.797 & 0.806 & 0.797 & 0.818 & \textbf{0.847} & 0.774 \\
			corral     & \textbf{1.000} & \textbf{1.000} & \textbf{1.000} & 0.998 & \textbf{1.000} & \textbf{1.000} \\
			crabs      & 0.974 & 0.976 & 0.977 & 0.967 & 0.983 & \textbf{0.997} \\
			hepatitis  & 0.790 & 0.793 & 0.782 & \textbf{0.868} & 0.848 & 0.793 \\
			biomed     & 0.927 & 0.945 & 0.947 & 0.961 & \textbf{0.980} & 0.953 \\
			haberman   & 0.669 & 0.687 & 0.699 & \textbf{0.708} & 0.697 & 0.667 \\
			heart-c    & 0.903 & 0.906 & 0.904 & 0.911 & \textbf{0.915} & 0.898 \\
			saheart    & 0.728 & 0.748 & 0.734 & 0.750 & \textbf{0.767} & 0.743 \\
			WDBC       & 0.987 & 0.993 & 0.993 & 0.990 & \textbf{0.996} & 0.991 \\
			EEG seizure & 0.990 & 0.995 & \textbf{0.996} & 0.993 & 0.991 & 0.987 \\
			\midrule
			Mean       & 0.830 & 0.838 & 0.836 & 0.837 & \textbf{0.853} & 0.831 \\
			\bottomrule
		\end{tabularx}
	\end{table}
	
	\begin{table}[h]
		\centering\small
		\caption{Panel-level Wilcoxon signed-rank tests on ROC~AUC across 12
			dataset means. Holm correction is applied across the five prespecified
			RPTE-versus-baseline comparisons.}
		\label{tab:wilcoxon}
		\renewcommand{\arraystretch}{1.08}
		\begin{tabularx}{\textwidth}{@{}l *{5}{>{\centering\arraybackslash}X}@{}}
			\toprule
			Comparison & $W$ & $p$ & $p_{\mathrm{Holm}}$ &
			Med.\ $\Delta$AUC & Record \\
			\midrule
			RPTE vs XGBoost  &  0 & 0.001 & 0.005$^*$ & $-0.006$ & 0-1-11 \\
			RPTE vs LightGBM & 16 & 0.148 & 0.442       & $-0.004$ & 3-1-8  \\
			RPTE vs RF       & 23 & 0.233 & 0.467       & $-0.006$ & 4-0-8  \\
			RPTE vs EBM      &  5 & 0.010 & 0.039$^*$ & $-0.020$ & 1-1-10 \\
			RPTE vs RuleFit  & 28 & 0.700 & 0.700       & $-0.001$ & 5-1-6  \\
			\bottomrule
		\end{tabularx}
		
		{\scriptsize $^*$Significant at $\alpha=0.05$ after Holm correction.
			Record = wins--ties--losses for RPTE.}
	\end{table}
	
	\subsection{Audit-load performance}\label{sec:auditLoad}
	Table~\ref{tab:complexity} reports MIU and compression ratios. Relative to opaque ensembles, RPTE achieves lower MIU on all 12 datasets, with XGBoost-to-RPTE ratios ranging from 9$\times$ (saheart) to 87$\times$ (corral). Relative to EBM, RPTE maintains lower MIU on all 12 datasets, with ratios ranging from 2$\times$ on postop.\ to 328$\times$ on EEG-Seizure. EBM's audit load grows substantially on feature-rich datasets because each feature generates a shape function with many bins: on EEG-Seizure (178 features), EBM produces 145,155 MIU. Relative to RuleFit, RPTE achieves lower MIU on 9 of 12 datasets. RuleFit achieves the lowest MIU on heart-c (64 vs. 75), saheart (62 vs. 91), and EEG-Seizure (152 vs. 442), where its compact rule sets produce fewer conditions. However, RuleFit does not enforce source disjointness, so its rules can share variables and fire simultaneously with conflicting signs (Appendix~\ref{app:critique}).
	
	\begin{table}[h]
		\centering\small
		\caption{Audit-load metrics: MIU, mean over 15 splits. Bold marks the lowest MIU per row. Parentheses report the MIU ratio baseline/RPTE; values above one favor RPTE.}
		\label{tab:complexity}
		\renewcommand{\arraystretch}{1.08}
		\begin{tabularx}{\textwidth}{@{}l *{6}{>{\centering\arraybackslash}X}@{}}
			\toprule
			Dataset & RPTE & XGBoost & LightGBM & RF & EBM & RuleFit \\
			\midrule
			lupus      & \textbf{10}  & 760\,(80$\times$)  &  744\,(78$\times$) &  2{,}807\,(294$\times$) & 159\,(17$\times$) & 59\,(6$\times$) \\
			postop.    & \textbf{16}  & 211\,(13$\times$)  &   73\,(5$\times$) &  4{,}587\,(294$\times$) &  33\,(2$\times$) & 98\,(6$\times$) \\
			append.    & \textbf{20}  & 660\,(34$\times$)  &  992\,(51$\times$) &  1{,}607\,(82$\times$) & 657\,(34$\times$) & 75\,(4$\times$) \\
			corral     & \textbf{18}  & 1{,}600\,(87$\times$) & 3{,}074\,(167$\times$) & 3{,}455\,(188$\times$) & 32\,(2$\times$) & 51\,(3$\times$) \\
			crabs      & \textbf{39}  & 1{,}894\,(49$\times$) & 4{,}544\,(118$\times$) & 7{,}557\,(196$\times$) & 1{,}033\,(27$\times$) & 55\,(1.44$\times$) \\
			hepatitis  & \textbf{40}  & 600\,(15$\times$)  &  497\,(13$\times$) &  3{,}761\,(95$\times$) & 368\,(9$\times$) & 80\,(2$\times$) \\
			biomed     & \textbf{35}  & 1{,}624\,(46$\times$) & 3{,}530\,(101$\times$) & 7{,}206\,(206$\times$) & 1{,}110\,(32$\times$) & 106\,(3$\times$) \\
			haberman   & \textbf{36}  & 775\,(22$\times$)  &  831\,(23$\times$) &  6{,}718\,(188$\times$) & 165\,(5$\times$) & 71\,(2$\times$) \\
			heart-c    & 75  & 1{,}247\,(17$\times$) & 1{,}193\,(16$\times$) & 6{,}692\,(90$\times$) & 370\,(5$\times$) & \textbf{64}\,(0.86$\times$) \\
			saheart    & 91  & 820\,(9$\times$)  & 1{,}297\,(14$\times$) & 4{,}462\,(49$\times$) & 1{,}492\,(16$\times$) & \textbf{62}\,(0.69$\times$) \\
			WDBC       & \textbf{46}  & 1{,}478\,(32$\times$) & 4{,}071\,(88$\times$) & 5{,}101\,(110$\times$) & 12{,}593\,(272$\times$) & 63\,(1.36$\times$) \\
			EEG seizure & 442 & 9{,}608\,(22$\times$) & 18{,}080\,(41$\times$) & 51{,}150\,(116$\times$) & 145{,}155\,(328$\times$) & \textbf{152}\,(0.34$\times$) \\
			\midrule
			Mean       & \textbf{72}  & 1{,}773 & 3{,}244 & 8{,}759 & 13{,}597 & 78 \\
			\bottomrule
		\end{tabularx}
	\end{table}
	
	Table~\ref{tab:pareto} presents the accuracy--auditability tradeoff for each method using both means and medians. RPTE combines high predictive accuracy with the lowest model-level audit load: its MIU is 72 (median 37), compared with 78 (68) for RuleFit, which lacks RPTE's source-disjointness guarantee. EBM achieves higher AUC and comparable prediction-level IIU, but its complete model-level MIU is approximately \(189\times\) higher by the mean and \(14\times\) higher by the median. The opaque ensembles improve mean AUC by only 0.5--0.8 percentage points over RPTE while requiring approximately \(25\times\)--\(121\times\) greater mean MIU and \(9\times\)--\(11\times\) greater mean IIU.
	
	\begin{table}[h]
		\centering\small
		\caption{Accuracy--auditability tradeoff across 12 datasets. Entries are mean (median) of the dataset-level results over 15 splits.}
		\label{tab:pareto}
		\renewcommand{\arraystretch}{1.08}
		\setlength{\tabcolsep}{4pt}
		\begin{tabularx}{\textwidth}
			{@{}l *{6}{>{\centering\arraybackslash}X}@{}}
			\toprule
			Metric & RPTE & XGBoost & LightGBM & RF & EBM & RuleFit \\
			\midrule
			AUC
			& 0.830 (.850)
			& 0.838 (.856)
			& 0.836 (.851)
			& 0.837 (.889)
			& 0.853 (.881)
			& 0.831 (.845) \\
			
			MIU
			& 72 (37)
			& 1{,}773 (1{,}034)
			& 3{,}244 (1{,}245)
			& 8{,}759 (4{,}844)
			& 13{,}597 (514)
			& 78 (68) \\
			
			IIU
			& 30.0 (14.4)
			& 279.2 (216.9)
			& 317.5 (180.3)
			& 341.7 (321.5)
			& 28.7 (14.5)
			& 78.1 (67.5) \\
			\bottomrule
		\end{tabularx}
	\end{table}
	
	\subsection{Impact of strict source ownership}
	\label{sec:strict}
	
	The default RPTE enforces source disjointness at the tree-ensemble level. Direct terms may share raw sources with trees or with each other. To quantify the effect of extending source ownership to direct terms, we evaluate a strict source-disjoint RPTE variant. After tree construction, direct terms are admitted greedily only when their raw sources are not already owned by a tree or an earlier admitted direct term.
	
	Table~\ref{tab:strict} reports the results across all 12 datasets (15 folds each, 180 folds total). The strict constraint reduces mean AUC by 1.8 percentage points. Active direct terms drop from 2,171 (default) to 243 (strict) across 180 folds, an 89\% reduction. 
	
	\begin{table}[h]
		\centering\small
		\caption{Strict source-disjointness analysis over 15 outer splits. Default RPTE is compared with the independently nested-CV-retuned strict variant.}
		\label{tab:strict}
		\renewcommand{\arraystretch}{1.08}
		\begin{tabularx}{\textwidth}{@{}l *{6}{>{\centering\arraybackslash}X}@{}}
			\toprule
			Dataset & AUC & AUC & $\Delta$AUC & MIU & MIU & $\Delta$MIU \\
			& default & strict & & default & strict & \\
			\midrule
			lupus     & 0.768 & 0.750 & $-0.018$ &  10 &  11 & $+1$ \\
			postop.   & 0.432 & 0.454 & $+0.022$ &  16 &  13 & $-3$ \\
			append.   & 0.797 & 0.751 & $-0.046$ &  20 &  16 & $-4$ \\
			corral    & 1.000 & 1.000 & 0.000  &  18 &  10 & $-8$ \\
			crabs     & 0.974 & 0.927 & $-0.047$ &  39 &  40 & $+1$ \\
			hepatitis & 0.790 & 0.788 & $-0.002$ &  40 &  34 & $-5$ \\
			biomed    & 0.927 & 0.913 & $-0.014$ &  35 &  29 & $-6$ \\
			haberman  & 0.669 & 0.645 & $-0.024$ &  36 &  31 & $-5$ \\
			heart-c   & 0.903 & 0.872 & $-0.031$ &  75 &  54 & $-20$ \\
			saheart   & 0.728 & 0.681 & $-0.047$ &  91 &  70 & $-21$ \\
			WDBC      & 0.987 & 0.981 & $-0.006$ &  46 &  46 & $0$ \\
			EEG seizure & 0.990 & 0.989 & $-0.001$ & 442 & 375 & $-68$ \\
			\midrule
			Mean      & 0.830 & 0.813 & $-.018$ &  72 &  61 & $-11$ \\
			\bottomrule
		\end{tabularx}
	\end{table}
	
	Panel-level Wilcoxon signed-rank tests on the 12 paired dataset means confirm that the AUC reduction under the strict constraint is statistically significant (AUC: $W = 6$, $p = 0.014$, Holm-corrected $p = 0.024$; balanced accuracy: $W = 4$, $p = 0.007$, Holm-corrected $p = 0.021$; Brier score: $W = 8$, $p = 0.012$, Holm-corrected $p = 0.024$). The additional term level strict constraint is therefore best suited to deployment scenarios where complete variable ownership (or reduced audit complexity) is desired and a modest accuracy cost is acceptable.

	\subsection{Ablation experiments}
	\label{sec:ablation}

	\begin{table}[t]
		\centering\small
		\caption{Ablation analysis across eight configurations and 12 datasets: mean ROC~AUC, $\Delta$AUC relative to full RPTE, and audit-load metrics.}
		\label{tab:ablation}
		\renewcommand{\arraystretch}{1.12}
		\setlength{\tabcolsep}{3pt}
		\begin{tabularx}{\textwidth}
			{@{}l *{4}{>{\centering\arraybackslash}X}@{}}
			\toprule
			Configuration & AUC & $\Delta$AUC & MIU & IIU \\
			\midrule
			\multicolumn{5}{@{}l}{\textit{RPTE estimator}} \\[2pt]
			Full RPTE (HUG + augmented pairs)
			& \textbf{0.830} & --       & 72  & 30.0 \\
			HUG, no augmentation
			& 0.819          & $-$0.011 & 62  & 18.8 \\
			Frequent top-$K$ vocabulary
			& 0.825          & $-$0.005 & 131 & 25.7 \\
			Original features only
			& 0.805          & $-$0.025 & 127 & 15.8 \\
			No leaf reassignment
			& 0.809          & $-$0.022 & 120 & 12.6 \\
			\midrule
			\multicolumn{5}{@{}l}{\textit{Fixed XGBoost on constructed vocabulary}} \\[2pt]
			HUG vocabulary + XGBoost
			& \textbf{0.832} & $+$0.001 & 1{,}540 & 264.3 \\
			HUG, no augmentation + XGBoost
			& 0.828          & $-$0.002 & 1{,}661 & 273.3 \\
			Frequent vocabulary + XGBoost
			& 0.827          & $-$0.003 & 1{,}595 & 267.5 \\
			\bottomrule
		\end{tabularx}
		
		\vspace{2pt}
		\begin{minipage}{\textwidth}
			\footnotesize
			AUC differences are calculated from the unrounded dataset-panel means. 
		\end{minipage}
	\end{table}
	
	\begin{table}[t]
		\centering\small
		\caption{Panel-level Wilcoxon signed-rank tests for the nine prespecified AUC contrasts. Holm correction is applied across all nine tests.	Record denotes wins--ties--losses for the first-named configuration.}
		\label{tab:ablation_wilcoxon}
		\renewcommand{\arraystretch}{1.08}
		\setlength{\tabcolsep}{4pt}
		\begin{tabularx}{\textwidth}
			{@{}X *{4}{>{\centering\arraybackslash}p{1.6cm}}@{}}
			\toprule
			Comparison & $W$ & $p$ & $p_{\mathrm{Holm}}$ & Record \\
			\midrule
			Full RPTE vs original only
			& 8 & 0.0122 & 0.0977 & 11--0--1 \\
			
			Full RPTE vs no reassignment
			& 5 & 0.0098 & 0.0879 & 10--1--1 \\
			
			Full RPTE vs HUG no augmentation
			& 23 & 0.2334 & 0.9336 & 8--0--4 \\
			
			HUG no augmentation vs frequent RPTE
			& 33 & 0.6772 & 1.000 & 7--0--5 \\
			
			Frequent RPTE vs frequent XGBoost
			& 15 & 0.1230 & 0.6152 & 2--1--9 \\
			
			HUG no-augmentation RPTE vs corresponding XGBoost
			& 16 & 0.0771 & 0.4629 & 2--0--10 \\
			
			Full RPTE vs HUG-vocabulary XGBoost
			& 11 & 0.0537 & 0.3760 & 1--1--10 \\
			
			HUG-vocabulary XGBoost vs HUG no-augmentation XGBoost
			& 29 & 0.7646 & 1.000 & 4--1--7 \\
			
			HUG no-augmentation XGBoost vs frequent XGBoost
			& 27 & 0.6377 & 1.000 & 6--1--5 \\
			\bottomrule
		\end{tabularx}
	\end{table}
	
	To isolate the contribution of each architectural choice, we conduct a systematic eight-scenario ablation using the complete \(3\times5\) repeated outer-CV protocol (12 datasets, 15 splits each, with identical folds and seeds). Five configurations evaluate RPTE under different vocabulary and reconstruction choices. Three additional configurations replace RPTE with fixed-hyperparameter XGBoost while retaining the vocabulary produced by the corresponding RPTE configuration, thereby comparing downstream learners on matched representations. Table~\ref{tab:ablation} reports the descriptive results, while Table~\ref{tab:ablation_wilcoxon} reports the nine prespecified panel-level comparisons. Per-dataset results appear in Appendix~\ref{app:ablation_detail}.	Three principal findings emerge from the RPTE configurations. First, Stage~1 vocabulary construction makes an important contribution to predictive accuracy: replacing the supervised vocabulary with original features reduces mean AUC by 2.5 percentage points, with full RPTE outperforming the original-only variant on 11 of 12 datasets. This difference is nominally significant but does not cross the \(5\%\) threshold after Holm correction across the nine contrasts (\(p_{\mathrm{Holm}}=.0977\)). Second, the HUG no-augmentation vocabulary produces a substantially more compact representation than the independently constructed frequent vocabulary. No significant AUC difference is detected between these configurations (0.819 vs.\ 0.825; \(p_{\mathrm{Holm}}=1.0\)), whereas HUG reduces mean MIU from 131 to 62, a \(2.1\times\) reduction that remains significant after correction (\(p_{\mathrm{Holm}}=.005\)). Third, leaf reassignment and the Stage~3 global refit improve mean AUC by 2.2 percentage points relative to retaining the Stage~2 leaf values, with full RPTE winning on 10 datasets, tying on one, and losing on one. The XGBoost configurations show that the constructed vocabularies can also transfer predictive information to an opaque learner. Their mean AUCs range from 0.827 to 0.832 and do not differ significantly from their prespecified RPTE counterparts after correction. However, they require approximately \(21\)--\(23\times\) the mean MIU and about \(9\times\) the mean IIU of full RPTE. These results indicate that vocabulary construction alone does not explain RPTE’s low audit load. The observed compactness is consistent with the combined effect of source-disjoint tree construction and Stage~3 global coefficient estimation.

	\subsection{Representation refinement and transfer}
	\label{sec:representation}
	
	The ablation demonstrates that HUG and frequent vocabularies achieve comparable AUC when paired with the same RPTE estimator. We investigate whether this performance parity arises because the two vocabularies produce similar data-level representations, or because the RPTE absorbs vocabulary differences during tree construction.
	
	\textit{Representation refinement.}
	We measure linear centered kernel alignment (CKA; \citealp{kornblith2019similarity}) between the HUG and frequent vocabularies at two stages. Stage-1 CKA compares only the binary mined-pattern matrices (excluding shared originals, augmented pairs, and labels). Pre-LR CKA compares the full design matrices supplied to Stage~3 logistic regression, including terminal-leaf indicators and eligible direct terms. On the 11 datasets where both vocabularies produce binary patterns, mean Stage-1 CKA is $0.554$. After RPTE tree construction, mean pre-LR CKA rises to $0.694$, an increase of $+0.140$ observed on all 11 eligible datasets. Fig.~\ref{fig:representation_refinement} shows the per-dataset values. The RPTE therefore absorbs a substantial portion of the initial vocabulary variation during structure discovery, bringing the two representations closer together before the final estimation step. The remaining CKA gap indicates that the representations are not identical; the HUG vocabulary retains a structurally distinct encoding that produces the same predictive performance at lower audit cost.
		
	\textit{Vocabulary transfer to XGBoost.}
	To assess whether the benefits of the HUG vocabulary extend beyond the 	RPTE estimator, we compare fixed-hyperparameter XGBoost trained on original features with the same XGBoost trained on the full HUG vocabulary
	(Table~\ref{tab:ablation}, bottom panel; Fig.~\ref{fig:vocabulary_transfer}). The two configurations achieve comparable observed AUC ($0.832$ vs.\ $0.824$, $p_{\mathrm{Holm}}=1.0$), while HUG-vocabulary XGBoost reduces MIU on 11 of 12 datasets and IIU on 9 of 12. 
	
	\begin{figure}[h]
		\includegraphics[width=0.7\textwidth,height=7cm]
		{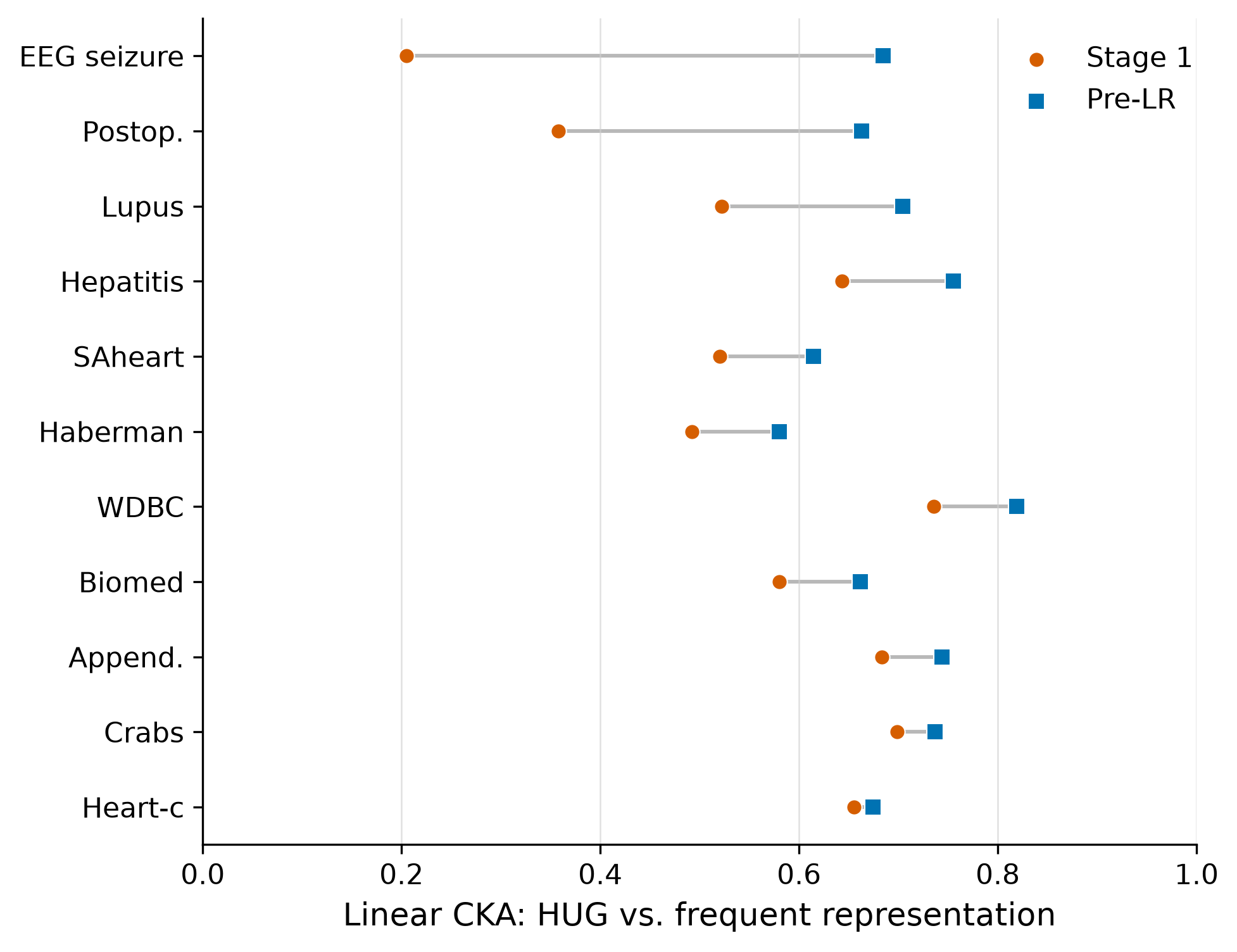}
		\caption{Linear CKA between HUG and frequent representations at Stage~1
			(mined-pattern matrices) and immediately before logistic-regression fitting
			(full RPTE design matrices).}
		\label{fig:representation_refinement}
	\end{figure}
	
	\begin{figure}[h]
		\includegraphics[width=0.7\textwidth,height=7cm]
		{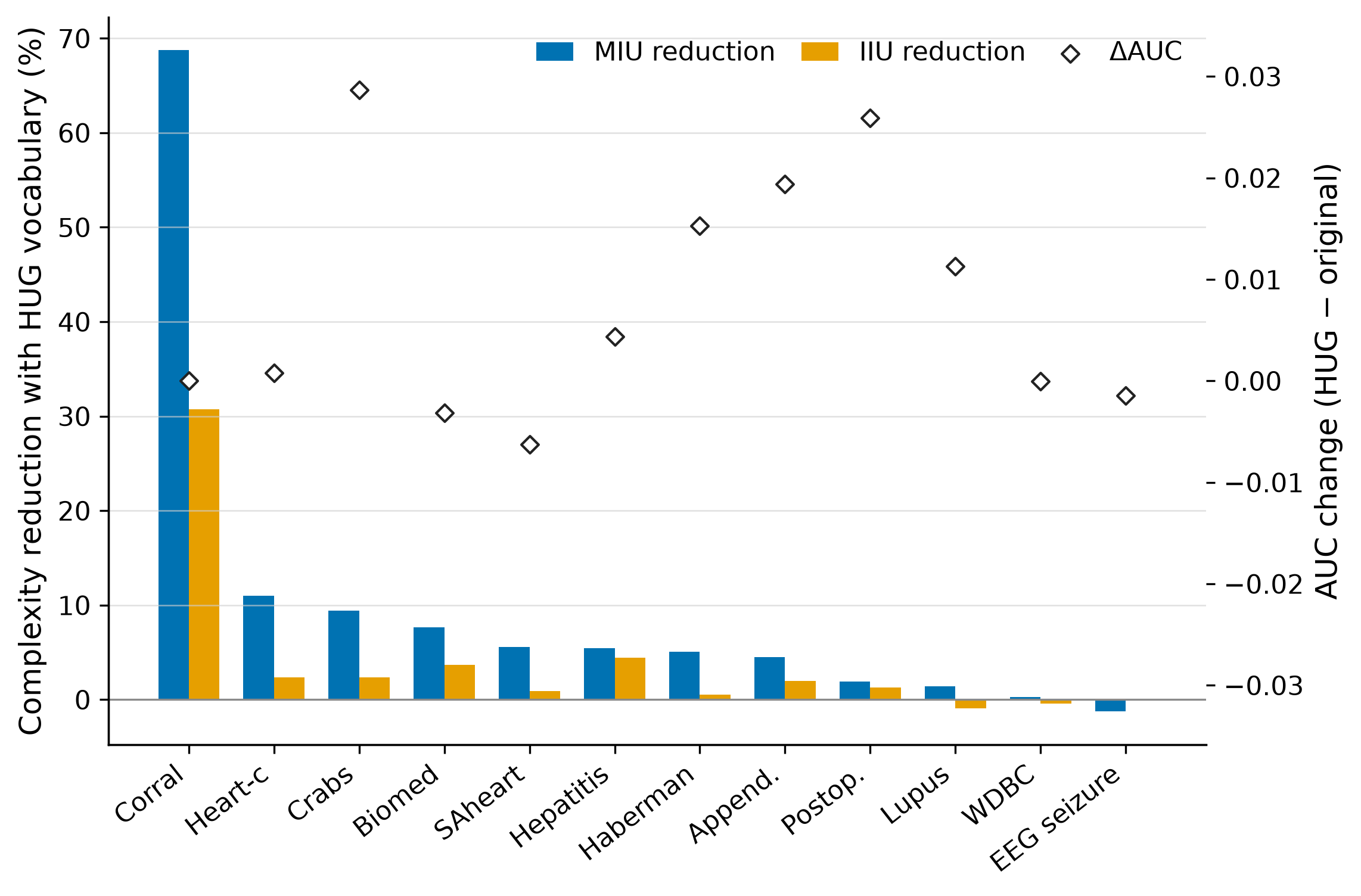}
		\caption{Impact of HUG vocabulary transfer on XGBoost.}
		\label{fig:vocabulary_transfer}
	\end{figure}
	
	A panel-level Wilcoxon test for MIU yields $p=.003$ ($p_{\mathrm{Holm}}=.068$); the corresponding IIU test yields $p=.009$ ($p_{\mathrm{Holm}}=.158$). Thus, the vocabulary primarily provides directional evidence of improved representational simplicity rather than a statistically confirmed improvement in predictive accuracy. The lower MIU and IIU are consistent with simpler fitted XGBoost structures when pre-composed patterns capture relevant interactions.

	\subsection{Execution time performance}
	
	Table~\ref{tab:timing} reports mean wall-clock execution times across all 12 datasets. RPTE tuning time is comparable to XGBoost and RF and substantially faster than EBM and RuleFit. EBM tuning averages 382 seconds due to iterative boosting over per-feature shape functions. RuleFit tuning averages 112 seconds.
	
	\begin{table}[h]
		\centering\small
		\caption{Execution time: mean (SD) of the 12 dataset-level mean times, in seconds.}
		\label{tab:timing}
		\renewcommand{\arraystretch}{1.08}
		\begin{tabularx}{\textwidth}{@{}l *{6}{>{\centering\arraybackslash}X}@{}}
			\toprule
			& RPTE & XGBoost & LightGBM & RF & EBM & RuleFit \\
			\midrule
			Tuning   & 26.2\,(62.4) & 16.0\,(29.2) & 11.7\,(36.8) & 9.7\,(17.3) & 382.1\,(688.4) & 112.2\,(111.5) \\
			Fit      & 1.987\,(5.431) & 0.783\,(1.510) & 0.464\,(1.452) & 0.727\,(1.786) & 24.551\,(52.971) & 6.009\,(9.672) \\
			Predict  & 0.019\,(0.041) & 0.003\,(0.003) & 0.005\,(0.013) & 0.012\,(0.009) & 0.004\,(0.012) & 0.090\,(0.049) \\
			\bottomrule
		\end{tabularx}
	\end{table}

	\section{Discussion, Implications, and Limitations}
	\label{sec:discussion}
	\subsection{Predictive accuracy under structural constraints (RQ1)}
	
	RPTE achieves a mean AUC of \(0.830\), \(0.8\) percentage points below the best opaque baseline, XGBoost (\(0.838\)). EBM leads with \(0.853\), while RuleFit has a similar panel mean of \(0.831\). Panel-level Wilcoxon tests across the five pre-specified RPTE-versus-baseline comparisons (Table~\ref{tab:wilcoxon}) show that RPTE performs significantly below XGBoost and EBM after Holm correction (\(p_{\mathrm{Holm}}=.005\) and \(.039\), respectively). Its differences from LightGBM, RF, and RuleFit are not significant (\(p_{\mathrm{Holm}}=.442\), \(.467\), and \(.700\)). RPTE nevertheless achieves perfect AUC on corral. On EEG-Seizure (\(n=11{,}500\), \(p=178\)), it achieves \(0.990\), only \(0.006\), or \(0.6\) AUC percentage points, below LightGBM. Overall, RPTE trades some predictive capacity for a substantially smaller, source-disjoint representation. It delivers competitive performance with markedly lower model-level audit complexity.
	
	\subsection{Audit-load reduction (RQ2)}
	
	The reduction in audit complexity is substantial against opaque ensembles and EBM. The XGBoost-to-RPTE MIU ratio ranges from 9$\times$ to 87$\times$ across datasets. RPTE maintains lower MIU than EBM on all 12 datasets. EBM's audit load grows with the number of features: on WDBC (30 features), EBM MIU reaches 12,593; on EEG-Seizure (178 features), it reaches 145,155, exceeding even RF (51,150). This occurs because EBM represents each feature with a shape function containing many bins, and pairwise interactions add further terms.
	
	The comparison with RuleFit is more nuanced. RuleFit achieves comparable aggregate MIU (mean 78 vs. 72) and lower MIU on three datasets (heart-c, saheart, EEG-Seizure). The key difference is structural rather than quantitative: RuleFit does not enforce source disjointness, so its rules can share variables, leading to the overlapping effects documented in Appendix~\ref{app:critique}. A RuleFit model with MIU = 64 may activate 20 rules per prediction with shared variables and conflicting coefficient signs, whereas an RPTE model with MIU = 75 produces non-overlapping contributions that can be read independently.
	
	We note that MIU and IIU measure condition counts rather than cognitive difficulty. Auditing a model with MIU = 75 is qualitatively different from reviewing one with MIU = 1,247: the former can be read entirely, while the latter requires automated tooling.
	
	\subsection{Interpretability analysis: Discovered patterns and domain alignment}
	\label{sec:interpretability}
	
	A central premise of this work is that RPTE discovers meaningful structure. We assess this by examining rules involving more than two raw variables that are selected across 15 outer folds in three datasets. Table~\ref{tab:patterns} summarizes the top recurring pattern signatures. Each dataset's pattern signatures are organized by raw-variable sets: a signature such as \{cp, oldpeak, thal\} means a rule whose split conditions involve exactly those three variables. The "Folds" column counts the number of folds in which that specific variable combination appears as a prediction-active rule. The "Domain concepts" column identifies the clinical concept families that the selected variables span. 
	
	\begin{table}[b]
		\centering\small
		\caption{Top recurring RPTE pattern signatures and domain alignment. Folds = number of folds (out of 15) in which the exact variable combination is prediction-active.}
		\label{tab:patterns}
		\renewcommand{\arraystretch}{1.12}
		\setlength{\tabcolsep}{3pt}
		\begin{tabularx}{\textwidth}{@{}l l c X@{}}
			\toprule
			Dataset & Pattern signature & Folds & Domain concepts matched \\
			\midrule
			\multirow{3}{*}{Heart-c}
			& \{cp, oldpeak, thal\}              & 4/15 & Ischaemia + anatomical evidence \\
			& \{ca, cp, oldpeak, sex, slope\}     & 3/15 & Ischaemia + anatomical + sex \\
			& \{ca, cp, oldpeak, thal, thalach, trestbps\} & 2/15 & Ischaemia + anatomical + baseline \\
			\midrule
			\multirow{3}{*}{WDBC}
			& \{mean\_concavity, worst\_area\}          & 3/15 & Nuclear size + boundary \\
			& \{worst\_concave\_pts, worst\_perim., worst\_texture\} & 3/15 & Nuclear size + boundary \\
			& \{mean\_concave\_pts, worst\_perim.\}     & 3/15 & Nuclear size + boundary \\
			\midrule
			\multirow{3}{*}{SA Heart}
			& \{Age, Famhist, Ldl, Tobacco, Typea\} & 10/15 & Lipid/tobacco + age/family \\
			& \{Age, Famhist, Ldl, Tobacco\}         &  8/15 & Lipid/tobacco + age/family \\
			& \{Age, Famhist, Ldl\}                  &  5/15 & Lipid + age/family \\
			\bottomrule
		\end{tabularx}
	\end{table}
	
	\textit{Heart-c: cardiac risk pathways.}
	Across all 15 folds, every RPTE model selects rules involving features from at least two of three clinically recognized cardiac pathways: exercise/ischaemia response (oldpeak, thal, exang, slope, thalach), anatomical/symptom evidence (cp, ca, thal), and baseline risk (age, trestbps, chol, fbs). At the concept level, all three pathway families are represented in all 15 folds. At the signature level, the most frequent exact variable combination is \{cp, oldpeak, thal\}, appearing in 4 of 15 folds. Different folds produce rules that span the same clinical concepts but use different variable subsets. Active-tree stability of 0.74 and prediction Spearman of 0.93 indicate that the overall tree architecture and risk ranking are highly consistent across folds, even as the specific variable combinations vary.
	
    \textit{WDBC: breast cancer morphology.}
	All 15 folds produce rules involving features from two diagnostic dimensions: nuclear size (radius, perimeter, area) and nuclear boundary irregularity (concavity, concave points, compactness). Both concepts recur at the concept level in all 15 folds. Individual signature frequencies are lower (3/15) because correlated size features (worst\_radius, worst\_perimeter, worst\_area) are interchangeable representations, and the Stage~1 redundancy filter retains different representatives in different folds.
	
	\textit{SA Heart: coronary risk factors.}
	RPTE discovers a dominant signature \{Age, Famhist, Ldl, Tobacco, Typea\} in 10 of 15 folds, unifying two established coronary risk dimensions: lipid/tobacco exposure and age/family-history risk. Active-tree stability of 0.72 and prediction Spearman of 0.82 indicate a consistent decision architecture.

	\subsection{Alternate explanations through Rashomon analysis}
	\label{sec:rashomon}
	
	An important question for auditable models is whether the discovered structure is the only near-optimal explanation. We investigate this through two complementary analyses.
	
	\paragraph{Cross-fold alternate explanations.}
	Two rules from different folds are coverage-equivalent if they identify the same observations (Jaccard $\ge 0.75$) with the same coefficient sign, even when built from different features. Representative pairs are drawn from the 15 outer splits of the \(3\times5\) repeated-CV protocol; split identifiers run from 0 to 14 across repetitions. On WDBC, 733 direct and 172 complement pairs are found; on heart-c, 32 pairs. Table~\ref{tab:crossfold} shows representative examples, with the features that differ between Rule~A and Rule~B highlighted in \textbf{bold}.
	
	\begin{table}[h]
		\centering
		\caption{Cross-fold alternate explanation pairs. Each row shows two rules from different folds that identify nearly the same patient subgroup (coverage Jaccard $\ge 0.90$) using different features. Features unique to each rule are in \textbf{bold}.}
		\label{tab:crossfold}
		\renewcommand{\arraystretch}{1.15}
		\resizebox{\textwidth}{!}{%
			\begin{tabular}{@{}p{1.2cm}cp{6.5cm}p{6.5cm}c@{}}
				\toprule
				Dataset & Folds & Rule A & Rule B & Cov.~J \\
				\midrule
				\multirow{2}{*}{WDBC}
				& 0 vs 9
				& worst\_perim $\le$ 106.0 AND worst\_concave\_pts $\le$ 0.135 AND \textbf{radius\_err} $\le$ 0.596\newline
				{\scriptsize 3 features, coeff $= -2.52$}
				& worst\_perim $\le$ 106.0 AND \textbf{worst\_compact} $\times$ \textbf{mean\_area} $\le$ 201.0 AND \textbf{area\_err} $\le$ 46.9 AND worst\_concave\_pts $\le$ 0.135\newline
				{\scriptsize 5 features, coeff $= -1.64$}
				& .987 \\
				\cmidrule{2-5}
				& 5 vs 14
				& \textbf{worst\_radius} $\le$ 16.8 AND worst\_concave\_pts $\le$ 0.136 AND area\_err $\le$ 38.6 AND smooth\_err $>$ 0.003\newline
				{\scriptsize 4 features, coeff $= -2.84$}
				& \textbf{worst\_perim} $\le$ 111.5 AND worst\_concave\_pts $\le$ 0.136 AND area\_err $\le$ 38.6 AND smooth\_err $>$ 0.003\newline
				{\scriptsize 4 features, coeff $= -2.43$}
				& .984 \\
				\midrule
				\multirow{2}{*}{Heart-c}
				& 1 vs 4
				& $|$cp $-$ thal$|$ $>$ 2.5 AND oldpeak $>$ 0.65\newline
				{\scriptsize 3 features, coeff $= -1.57$}
				& cp$=$[0,1), oldpeak$=$[0.8, 6.2) AND \textbf{trestbps} $\times$ thal $>$ 290\newline
				{\scriptsize 4 features, coeff $= -1.42$}
				& .918 \\
				\cmidrule{2-5}
				& 9 vs 11
				& cp $-$ oldpeak $> -0.45$ AND ca $\le$ 0.5 AND {chol} $\times$ thal $\le$ 628.5\newline
				{\scriptsize 5 features, coeff $= +1.20$}
				& cp $-$ oldpeak $> -0.45$ AND ca $\le$ 0.5 AND \textbf{trestbps} $\times$ thal $\le$ 317.5 AND {chol} $\le$ 323\newline
				{\scriptsize 6 features, coeff $= +1.74$}
				& .902 \\
				\bottomrule
		\end{tabular}}
	\end{table}
	
	On WDBC, the first pair identifies nearly the same subgroup (coverage Jaccard \(=0.987\)) despite providing substantially different explanations: both rules retain worst perimeter and worst concave points, but one uses radius error whereas the other uses area error and a worst-compactness–mean-area product. The second pair replaces worst radius with worst perimeter while retaining the other conditions and achieves a coverage Jaccard of \(0.984\), illustrating how correlated tumor-size measures can support alternative explanations of nearly the same subgroup. On heart-c, different arithmetic and threshold conditions similarly provide alternative explanations that identify nearly the same patients and have the same coefficient direction.
	
	\paragraph{Same-fold alternate explanations.}
	To examine whether RPTE can produce different explanations under a fixed data partition, we refit the model after removing one raw feature at a time. Candidates are screened using three-fold out-of-fold predictions within the outer-training fold, and at most five alternatives per dataset are evaluated on the common held-out fold. All ten selected alternatives preserved held-out discrimination, with an AUC loss of at most \(0.01\). Three of five heart-c alternatives also met the joint prediction-agreement criteria, whereas the WDBC alternatives preserved AUC but showed greater changes in risk ranking and top-risk membership. Table~\ref{tab:samefold_rules} compares prediction-active leaf rules from the reference and alternative models. 
	
	\begin{table}[h]
		\centering
		\caption{Same-fold alternate explanations. Reference and alternative models are trained on the same outer-training fold using different raw feature sets. Differing conditions are shown in \textbf{bold}.}
		\label{tab:samefold_rules}
		\renewcommand{\arraystretch}{1.15}
		\resizebox{\textwidth}{!}{%
			\begin{tabular}{@{}
					p{1.2cm}
					>{\raggedright\arraybackslash}p{2.2cm}
					>{\raggedright\arraybackslash}p{6.4cm}
					>{\raggedright\arraybackslash}p{6.4cm}
					c@{}}
				\toprule
				Dataset & Removed feature & Reference rule & Alternative rule & Cov.\ $J$ \\
				\midrule
				
				\multirow{2}{*}{WDBC}
				& \makecell[l]{\texttt{worst\_}\\\texttt{symmetry}}
				& worst\_perimeter $\le105.95$ AND
				worst\_concave\_points $\le0.13505$ AND
				\textbf{radius\_error $\le0.59555$}\newline
				{\scriptsize coefficient $=-2.52$}
				& worst\_perimeter $\le105.95$ AND
				worst\_concave\_points $\le0.13505$ AND
				\textbf{area\_error $\le44.13$}\newline
				{\scriptsize coefficient $=-2.45$}
				& 1.000 \\
				
				\cmidrule{2-5}
				
				& \makecell[l]{\texttt{worst\_}\\\texttt{perimeter}}
				& \textbf{worst\_perimeter $\le105.95$} AND
				worst\_concave\_points $\le0.13505$ AND
				\textbf{radius\_error $\le0.59555$}\newline
				{\scriptsize coefficient $=-2.52$}
				& \textbf{worst\_radius $\le16.805$} AND
				worst\_concave\_points $\le0.13505$ AND
				\textbf{area\_error $\le48.7$} AND
				\textbf{smoothness\_error $>0.003294$} AND
				\textbf{worst\_texture $\le33.27$}\newline
				{\scriptsize coefficient $=-2.77$}
				& .857 \\
				
				\midrule
				
				\multirow{2}{*}{Heart-c}
				& \texttt{fbs}
				& cp $-$ oldpeak $\le-0.7$ AND
				\textbf{thal $>2.5$}\newline
				{\scriptsize coefficient $=-1.90$}
				& cp $-$ oldpeak $\le-0.7$ AND
				\textbf{trestbps $\times$ thal $>315$}\newline
				{\scriptsize coefficient $=-1.93$}
				& .944 \\
				
				\cmidrule{2-5}
				
				& \texttt{exang}
				& cp $-$ oldpeak $\le-0.7$ AND thal $\le2.5$ AND
				\textbf{NOT(ca$=[1,4)$, cp$=[0,1)$)} AND
				trestbps $-$ thalach $\le-8.5$\newline
				{\scriptsize coefficient $=+0.17$}
				& cp $-$ oldpeak $\le-0.7$ AND thal $\le2.5$ AND
				\textbf{ca$=[0,1)$} AND
				trestbps $-$ thalach $\le-8.5$\newline
				{\scriptsize coefficient $=+0.46$}
				& 1.000 \\
				
				\bottomrule
		\end{tabular}}
	\end{table}
	
	The same-fold comparisons reveal both feature substitution and structural persistence. On WDBC, radius error and area error produce coverage-equivalent explanations when combined with the same perimeter and concavity conditions (\(J=1.000\)). When worst perimeter is removed, RPTE reconstructs the explanation using worst radius and related morphology features while retaining high subgroup agreement (\(J=.857\)). On heart-c, a thal threshold is replaced by a trestbps–thal product with closely aligned coverage and coefficient magnitude (\(J=.944\)); the second pair achieves exact coverage equivalence through an alternative ca condition (\(J=1.000\)).
	
	Unlike the cross-fold analysis, these comparisons hold the training and test partitions fixed. The observed alternatives therefore arise from fitting RPTE with different feature sets rather than from changes in sampled patients. This demonstrates that RPTE can expose multiple, directly comparable explanations for nearly the same patient groups while preserving effect direction and predictive behavior.

	\subsection{Practical implications and limitations}
	
	The results have practical implications for predictive modeling in regulated settings. On heart-c, RPTE has an MIU of approximately 75, making comprehensive human review more tractable. By comparison, XGBoost has an MIU of 1,247 on heart-c, while EBM reaches 145,155 on EEG-Seizure, making automated model-level inspection increasingly necessary. Other interpretable methods, including RuleFit, SIRUS, FIGS, and EBM, provide algebraically additive predictions. However, additivity alone does not ensure distinct explanatory components: rules or tree contributions may overlap, form nested conditions, or reuse the same variables across multiple terms. RPTE instead organizes predictions around source-disjoint trees, with each observation activating one leaf per tree. A prediction can therefore be traced exactly to a compact set of named leaf rules and direct terms.
	
	The ablation results (Section~\ref{sec:ablation}) provide guidance for deployment. Stage~1 vocabulary construction contributes the largest accuracy gain and should not be omitted. The HUG-based vocabulary produces substantially more compact models than the independently constructed frequent vocabulary at comparable AUC, making it the preferred Stage~1 strategy when audit load matters. Source disjointness imposes a modest accuracy cost for a substantial reduction in audit load, a tradeoff favorable in most regulated settings. The representation analysis (Section~\ref{sec:representation}) confirms that the HUG vocabulary transfers structural benefits even to opaque learners: when supplied to a fixed-hyperparameter XGBoost, HUG reduces MIU on 11 of 12 datasets at comparable AUC. This indicates that the vocabulary captures genuine structural regularities rather than artifacts of the RPTE estimator. The strict source-disjointness analysis (Section~\ref{sec:strict}) further shows that complete variable ownership can be enforced at a mean cost of 1.8 AUC points, with an 89\% reduction in active direct terms.
	
	The proposed approach has several limitations that should be considered when interpreting the results. First, the statistical comparison is based on 12 dataset-level observations. After Holm correction across the five prespecified RPTE-versus-baseline tests, RPTE performs significantly below XGBoost and EBM (\(p_{\mathrm{Holm}}=.005\) and \(.039\)), while its differences from LightGBM, RF, and RuleFit are not significant.  Second, the same-fold Rashomon analysis examines a systematically generated, bounded set of alternatives rather than the complete Rashomon set, whose exhaustive enumeration would be combinatorially prohibitive. Third, MIU and IIU quantify structural audit load rather than human cognitive effort. Hence, user studies with domain experts are needed to establish how these reductions affect review time and decision quality. Finally, extensions to multi-class classification, survival analysis, and regression are natural next steps.
	
	\section{Conclusion and Future Research}
	\label{sec:conclusion}
	
	This paper introduced a three-stage learning approach that separates structure discovery from coefficient estimation in tabular classification. Coefficient reassignment and source-disjoint tree construction yield models that retain competitive predictive performance while substantially reducing audit complexity. The proposed approach achieves a mean ROC~AUC of 0.830 across twelve benchmarks, within 0.8 percentage points of the best opaque baseline, while reducing model inspection units by 9$\times$ to 87$\times$ relative to XGBoost and maintaining lower MIU than EBM on all 12 datasets. RuleFit achieves comparable or lower MIU on three datasets but without the source-disjointness guarantees. Five structural properties are formally established, providing the theoretical basis for audit-load guarantees. An eight-scenario ablation confirms that the supervised HUG vocabulary is the primary driver of both predictive accuracy and representational economy: it outperforms the original-features-only variant by 2.5 AUC points and produces $2.1\times$ lower MIU than an independently constructed frequent vocabulary at comparable AUC. The learned vocabulary also transfers to XGBoost, reducing MIU on 11 of 12 datasets at matched AUC, demonstrating that the vocabulary captures genuine structural regularities.
	
	The approach complements existing methods. Gradient-boosted ensembles remain the method of choice when maximum accuracy is the sole objective. EBMs provide an effective alternative when pairwise interactions suffice and per-feature shape functions are preferred. The three-stage approach targets scenarios where higher-order interactions matter and full per-prediction auditability is required.
	
	Several research directions follow from the modular design of RPTE. First, alternative Stage~1 vocabularies and structured Stage~3 penalties, including elastic net \citep{zou2005regularization} and group lasso, could clarify how representation and coefficient regularization interact. Second, scalable search procedures could characterize larger regions of the Rashomon set without requiring combinatorially infeasible exhaustive enumeration. Third, external, multi-site, and temporal validation would establish whether the discovered structures remain stable and transferable across populations. Finally, extensions that incorporate domain-specific constraints or prior knowledge could further improve the relevance and actionability of the resulting explanations.
	
	\bibliographystyle{apalike}
	\bibliography{AuditableClassifiers}
	
	
	\begin{appendices}

	\section{Interpretability Analysis of Additive Rule Methods}
	\label{app:critique}
	
	Individual rules in RuleFit, FIGS, and SIRUS are human-readable. However, the fitted rule collections can impose a substantial cognitive burden when a reviewer attempts to reconstruct any single prediction. This appendix examines three recurring patterns that widen the gap between rule-level readability and model-level interpretability. These are not mathematical errors. The concern is operational: how much mental work is required to trace a prediction.
	
	We refit all three methods on five datasets using the first outer fold (seed 42, stratified 5-fold CV). RuleFit uses the nested-CV best parameters; SIRUS uses 10 rules at depth 2; FIGS uses \texttt{max\_rules}$= 100$. For each fitted model, we extract all active rules and compute pairwise training-set coverage overlaps. Table~\ref{tab:critique} summarizes the results.
	
	\begin{enumerate}[label=(\arabic*),leftmargin=2.5em,itemsep=2pt]
		\item \textbf{Conflicting overlapping effects.}
		When rules on the same features fire simultaneously with opposite signs, no single rule can be read as the model's conclusion about those features.
		
		\item \textbf{Nested threshold families.}
		When multiple rules use the same variables at different thresholds and one rule's coverage is a subset of another's, their contributions accumulate with no single clear boundary.
		
		\item \textbf{Activation multiplicity.}
		The number of rules firing simultaneously per observation determines the cognitive load of tracing any prediction.
	\end{enumerate}
	
	\newcommand{\rr}[1]{{\small\texttt{#1}}}
	
	\begin{table}[H]
		\centering
		\caption{Interpretability issues observed in fitted additive rule models
			(outer fold 0, seed 42). For SIRUS, contrast denotes the true-branch output minus the false-branch output.}
		\label{tab:critique}
		\renewcommand{\arraystretch}{1.08}
		\setlength{\tabcolsep}{3pt}
		
		\begin{tabular}{@{}p{2.8cm}p{13.3cm}@{}}
			\toprule
			\multicolumn{2}{@{}l}{\textbf{RuleFit}} \\
			\midrule
			
			Conflicting overlap
			&
			\texttt{prnn\_crabs}:
			\rr{\texttt{index}$\leq23.5$, \texttt{FL}$\leq17.3$,
				\texttt{BD}$\leq11.45$}
			(coef.\ $=-1.23$, train $=38$)
			and
			\rr{\texttt{index}$\leq26.0$, \texttt{FL}$\leq17.75$,
				\texttt{BD}$\leq11.3$}
			(coef.\ $=+0.17$, train $=38$).
			The rules overlap on 36 training and 10 held-out observations
			(train/test Jaccard $=.90/.91$), so both opposing contributions apply
			to nearly the same subgroup.
			\\
			
			Nested thresholds
			&
			\texttt{prnn\_crabs}: seven \{\texttt{FL}, \texttt{index}\} rules form
			10 nested pairs. Example:
			\rr{\texttt{index}$\leq23.5$, \texttt{FL}$\leq17.1$} ($-0.87$)
			within
			\rr{\texttt{index}$\leq24.5$, \texttt{FL}$\leq18.05$} ($-0.11$);
			train/test Jaccard $=.949/.833$.
			\\[3pt]
			
			Activation multiplicity
			&
			Mean active rules per held-out observation: $6.7$--$19.2$;
			maximum $=35$ (\texttt{prnn\_crabs}).
			\\
			
			\midrule
			\multicolumn{2}{@{}l}{\textbf{FIGS}} \\
			\midrule
			
			Conflicting overlap
			&
			Hepatitis:
			\rr{\texttt{BILIRUBIN}$\leq13.5$,
				\texttt{ALBUMIN}$\leq28.5$} (Tree 1, $+0.192$, $n=87$)
			and
			\rr{\texttt{BILIRUBIN}$\leq16.5$,
				\texttt{ALBUMIN}$\leq28.5$} (Tree 3, $-0.022$, $n=91$);
			train/test Jaccard $=.956/.958$.
			\\[3pt]
			
			Nested thresholds
			&
			Hepatitis:
			\rr{\texttt{ASCITES}$>0.5$, \texttt{BILIRUBIN}$>18$,
				$29<\texttt{SGOT}\leq82.5$}
			contains
			\rr{$0.5<\texttt{ASCITES}\leq1.5$, \texttt{BILIRUBIN}$>18$,
				$37.5<\texttt{SGOT}\leq81.5$};
			10 shared training observations, Jaccard $=.909$.
			\\[3pt]
			
			Activation multiplicity
			&
			One leaf is active per tree; hence active leaves per prediction equal the number of trees, ranging from 1 to 18 across the five fitted models.
			\\
			
			\midrule
			\multicolumn{2}{@{}l}{\textbf{SIRUS}} \\
			\midrule
			
			Conflicting overlap
			&
			Postoperative:
			\rr{\texttt{BP-STBL}$<1$}
			(contrast $=-.091$, train/test $=15/6$)
			within
			\rr{\texttt{BP-STBL}$<2$}
			(contrast $=+.048$, train/test $=53/13$);
			train/test Jaccard $=.283/.462$.
			\\[3pt]
			
			Nested thresholds
			&
			Appendicitis:
			\rr{\texttt{At7}$<.1986$}
			(contrast $=+.499$, train/test $=17/4$)
			within
			\rr{\texttt{At7}$<.2778$}
			(contrast $=+.412$, train/test $=25/8$);
			train/test Jaccard $=.680/.500$.
			\\[3pt]
			
			Activation multiplicity
			&
			All 10 rule components contribute through true/false outputs;
			mean true-branch memberships $=2.27$--$3.83$, maximum $=6$--$9$.
			\\
			
			\bottomrule
		\end{tabular}
	\end{table}
	
	\begin{minipage}{.95\textwidth}
		\textit{Note.}
		Shared variables do not necessarily imply redundant or conflicting
		explanations; threshold, direct, and interaction effects may remain
		meaningfully distinct. The examples above concern repeated subgroup
		descriptions with overlapping or nested thresholds. RPTE separates source
		variables across trees and activates exactly one leaf within each tree.
		The strict variant additionally constrains retained direct terms.
	\end{minipage}
	
	\section{Hyperparameter Configuration}
	\label{app:params}
	
	All models are tuned using inner three-fold stratified cross-validation
	within each outer-training fold. For every model, the Cartesian product
	of three binary choices produces eight candidate configurations. The
	candidate with the highest mean inner-fold ROC~AUC is refitted on the
	complete outer-training fold.
	
	\begin{table}[h]
		\centering\small
		\caption{Hyperparameter search spaces. Each model has eight candidate configurations evaluated by inner three-fold stratified cross-validation.}
		\label{tab:hyperparameters}
		\renewcommand{\arraystretch}{1.14}
		\setlength{\tabcolsep}{5pt}
		\begin{tabularx}{\textwidth}
			{@{}l
				>{\hsize=1.2\hsize\raggedright\arraybackslash}X
				>{\hsize=0.8\hsize\raggedright\arraybackslash}X@{}}
			\toprule
			Method & Tuned parameters & Fixed settings \\
			\midrule
			RPTE
			& {
				$L\in\{1,2\}$, $\mathrm{topK}\in\{50,100\}$,$G\in\{0.01,0.001\}$,
				Sequential RPTE estimator
			}
			& Package defaults \\
			
			XGBoost
			&
			$\mathrm{max\_depth}\in\{2,4\}$;\;
			$\mathrm{learning\_rate}\in\{0.03,0.1\}$;\;
			$\mathrm{min\_child\_weight}\in\{1,5\}$
			&
			$\mathrm{reg\_lambda}=1$;\;
			$\mathrm{n\_estimators}=200$;\;
			early stopping after 20 rounds
			\\
			
			LightGBM
			&
			$\mathrm{learning\_rate}\in\{0.03,0.1\}$;\;
			$\mathrm{num\_leaves}\in\{15,31\}$;\;
			$\mathrm{min\_child\_samples}\in\{10,30\}$
			&
			$\mathrm{reg\_lambda}=1$;\;
			$\mathrm{n\_estimators}=200$
			\\
			
			Random Forest
			&
			$\mathrm{n\_estimators}\in\{50,100\}$;\;
			$\mathrm{max\_depth}\in\{4,8\}$;\;
			$\mathrm{min\_samples\_leaf}\in\{1,5\}$
			&
			$\mathrm{max\_features}=\sqrt{p}$
			\\
			
			EBM
			&
			$\mathrm{max\_rounds}\in\{200,500\}$;\;
			$\mathrm{interactions}\in\{0,10\}$;\;
			$\mathrm{learning\_rate}\in\{0.015,0.05\}$
			&
			Package defaults
			\\
			
			RuleFit
			&
			$\mathrm{n\_estimators}\in\{50,100\}$;\;
			$\mathrm{max\_rules}\in\{30,60\}$;\;
			$\mathrm{tree\_size}\in\{3,4\}$
			&
			Package defaults
			\\
			\bottomrule
		\end{tabularx}
	\end{table}

	\section{Detailed Ablation Results}
	\label{app:ablation_detail}
	
	Table~\ref{tab:ablation_rpte_detail} reports per-dataset AUC and MIU
	for the five RPTE configurations. All entries are means over the 15
	outer splits of the $3\times5$ repeated-CV protocol.
	
	The HUG-no-augmentation configuration achieves the lowest MIU on 11 of 12 datasets. Full RPTE achieves the highest AUC on seven datasets, including its tie with frequent RPTE and no reassignment on corral. The frequent vocabulary has higher MIU than both HUG configurations on 11 datasets. Overall, the HUG-based representations are substantially more compact while retaining comparable predictive performance.
	
	\begin{table}[h]
		\centering\scriptsize
		\caption{Per-dataset ablation results:
			mean ROC~AUC and MIU over 15 outer splits. Bold denotes the highest AUC
			and lowest MIU within each dataset, determined from unrounded values.}
		\label{tab:ablation_rpte_detail}
		\renewcommand{\arraystretch}{1.08}
		\setlength{\tabcolsep}{2pt}
		\begin{tabularx}{\textwidth}
			{@{}l *{10}{>{\centering\arraybackslash}X}@{}}
			\toprule
			& \multicolumn{2}{c}{Full RPTE}
			& \multicolumn{2}{c}{HUG no aug.}
			& \multicolumn{2}{c}{Frequent}
			& \multicolumn{2}{c}{Original only}
			& \multicolumn{2}{c}{No reassign.} \\
			\cmidrule(lr){2-3}
			\cmidrule(lr){4-5}
			\cmidrule(lr){6-7}
			\cmidrule(lr){8-9}
			\cmidrule(lr){10-11}
			Dataset
			& AUC & MIU
			& AUC & MIU
			& AUC & MIU
			& AUC & MIU
			& AUC & MIU \\
			\midrule
			lupus
			& \textbf{0.768} & 10
			& 0.758 & \textbf{8}
			& 0.755 & 39
			& 0.767 & 36
			& 0.754 & 39 \\
			
			postop.
			& 0.432 & 16
			& \textbf{0.479} & \textbf{6}
			& 0.475 & 35
			& 0.467 & 35
			& 0.388 & 39 \\
			
			append.
			& \textbf{0.797} & 20
			& 0.780 & \textbf{16}
			& 0.780 & 60
			& 0.760 & 61
			& 0.756 & 53 \\
			
			corral
			& \textbf{1.000} & 18
			& 0.883 & 24
			& \textbf{1.000} & 13
			& 0.883 & 43
			& \textbf{1.000} & \textbf{13} \\
			
			crabs
			& \textbf{0.974} & 39
			& 0.934 & \textbf{31}
			& 0.909 & 86
			& 0.904 & 85
			& 0.928 & 55 \\
			
			hepatitis
			& 0.790 & 40
			& \textbf{0.822} & \textbf{37}
			& 0.802 & 92
			& 0.777 & 110
			& 0.808 & 116 \\
			
			biomed
			& 0.927 & 35
			& \textbf{0.928} & \textbf{29}
			& 0.922 & 106
			& 0.915 & 101
			& 0.924 & 87 \\
			
			haberman
			& 0.669 & 36
			& 0.669 & \textbf{25}
			& \textbf{0.682} & 59
			& 0.660 & 50
			& 0.634 & 50 \\
			
			heart-c
			& \textbf{0.903} & 75
			& 0.897 & \textbf{63}
			& 0.902 & 112
			& 0.874 & 97
			& 0.871 & 89 \\
			
			saheart
			& \textbf{0.728} & 91
			& 0.708 & \textbf{73}
			& 0.703 & 118
			& 0.684 & 101
			& 0.669 & 99 \\
			
			WDBC
			& 0.987 & 46
			& 0.984 & \textbf{43}
			& \textbf{0.988} & 189
			& 0.983 & 184
			& 0.984 & 232 \\
			
			EEG seizure
			& \textbf{0.990} & 442
			& 0.987 & \textbf{384}
			& 0.987 & 661
			& 0.986 & 619
			& 0.988 & 567 \\
			\midrule
			Mean
			& \textbf{0.830} & 72
			& 0.819 & \textbf{62}
			& 0.825 & 131
			& 0.805 & 127
			& 0.809 & 120 \\
			\bottomrule
		\end{tabularx}
	\end{table}

\end{appendices}	
\end{document}